\documentclass{article}
\usepackage{./latex_styles_common}
\usepackage{fullpage}
\usepackage{xcolor}
\usepackage{url}
\usepackage{verbatim}
\usepackage{graphicx}
\usepackage{parskip}
\usepackage{booktabs}       % professional-quality tables
\usepackage{minipage-marginpar}
\usepackage{authblk}

\usepackage{nicefrac}
\usepackage{listings} % For displaying codek34
\usepackage{multirow}

\usepackage{algorithm}
\usepackage{bbm}
\usepackage[linesnumbered,noend,algo2e]{algorithm2e} 

\usepackage{subcaption}

\usepackage[nameinlink]{cleveref}
\Crefname{figure}{Figure}{Figures}
\crefname{figure}{Figure}{Figures}
\crefname{example}{Example}{Example}
\crefname{theorem}{Theorem}{Theorem}
\crefname{corollary}{Corollary}{Corollary}
\crefname{lemma}{Lemma}{Lemma}
\crefname{proposition}{Proposition}{Proposition}
\crefname{assumption}{Assumption}{Assumption}
\crefname{section}{Section}{Section}
\crefname{algorithm}{Algorithm}{Algorithm}

\usepackage{enumitem}
\newlist{propenum}{enumerate}{1} % also creates a counter called 'propenumi'
\setlist[propenum]{label=\alph*{\rm)}, ref=\theproposition(\alph*)}
\crefalias{propenumi}{proposition}
\newlist{corenum}{enumerate}{1} % also creates a counter called 'propenumi'
\setlist[corenum]{label=\alph*{\rm)}, ref=\thecorollary(\alph*)}
\crefalias{corenumi}{corollary}
\newlist{lemenum}{enumerate}{1} % also creates a counter called 'propenumi'
\setlist[lemenum]{label=\alph*{\rm)}, ref=\thelemma(\alph*)}
\crefalias{lemenumi}{lemma}

\usepackage{amsthm,thmtools}
\declaretheorem[name=Theorem,numberwithin=section]{theorem}
\declaretheorem[name=Definition,style=definition,numberlike=theorem]{definition}

\declaretheorem[name=Assumption,numberlike=theorem]{assumption}
\declaretheorem[name=Lemma,numberlike=theorem]{lemma}

\newcommand{\bL}{\boldsymbol{L}}

\newcommand{\wt}[1]{\widetilde{#1}}
\newcommand{\init}{\mathrm{init}}

\definecolor{blu}{rgb}{0,0,1}

\definecolor{gre}{rgb}{0,.5,0}

\definecolor{red}{rgb}{1,0,0}

\def\norm#1{\|#1\|}

\usepackage{todonotes}

\usepackage{lipsum}

\title{
  A dual approach for federated learning
}

\author[1]{Zhenan Fan\footnote{Contributed equally. Authors are listed in alphabetical order.}}
\newcommand\CoAuthorMark{\footnotemark[\arabic{footnote}]}
\author[2]{Huang Fang\protect\CoAuthorMark}
\author[1]{Michael P. Friedlander}
\affil[1]{University of British Columbia
        \texttt{ \{zhenanf, mpf\}@cs.ubc.ca} }
\affil[2]{Baidu Research \texttt{ fanghuang@baidu.com } }

\date{\today}

\begin{document}

  \maketitle

\begin{abstract}
    % Federated optimization stays at the core of federated learning. By looking at a dual formulation of the federated optimization problem,
    We study the federated optimization problem from a dual perspective and propose a new algorithm termed federated dual coordinate descent (FedDCD), which is based on a type of  coordinate descent method developed by Necora et al.~\emph{[Journal of Optimization Theory and Applications, 2017]}. Additionally, we enhance the FedDCD method with inexact gradient oracles and Nesterov’s acceleration. We demonstrate theoretically that our proposed approach achieves better convergence rates than the state-of-the-art primal federated optimization algorithms under certain situations. Numerical experiments on real-world datasets support our analysis.
\end{abstract}

\section{Introduction} \label{sec:intro}

With the development of artificial intelligence, people recognize that many powerful machine learning models are driven by large distributed datasets, e.g., AlphaGo~\citep{silver2016mastering} and AlexNet~\citep{krizhevsky2012imagenet}. In many industrial scenarios, training data are maintained by different organizations, and transporting or sharing the data across these organizations is not feasible because of regulatory and privacy considerations~\citep{li2020review}. Therefore there is increasing interest in training machine learning models that operate without gathering all data in a single place.
Federated learning (FL), initially proposed by~\citet{mcmahan2017communication} to train models on decentralized data from mobile devices, and later extended by~\citet{yang2019federated} and \citet{kairouz2019advances}, is a training framework that allows multiple clients to collaboratively train a model without sharing data. 

The learning process in FL can be formulated as a distributed optimization problem, which is also known as federated optimization (FO). Assume there are $N$ clients and each client $i$ maintains a local dataset $\Dscr_i$. FO aims to solve the  empirical-risk minimization problem
\begin{equation} \label{eq:primal}
  \minimize{w\in\Real^d}\enspace F(w) \coloneqq \sum_{i = 1}^N f_i(w) \tag{P}
\end{equation}
in a distributed manner, where $w$ is the global model parameter and each local objective $f_i:\Real^d \to\Real$ is defined by 
\begin{equation*} %\label{eq:local_objective}
    f_i(w) = \ell(w; \Dscr_i),
\end{equation*}
where $\ell(\cdot, \Dscr_i)$ is a convex and differentiable loss function for each $\Dscr_i$. 

As characterized and formalized by~\citet{wang2021field},~\citet{li2020federated} and~\citet{li2019convergence}, there are several important characteristics that distinguish FO from standard machine learning and distributed optimization. 
\begin{assumption}[Governing assumptions] \label{assum:govern}
The following assumptions hold for FO. 
  \begin{itemize}
    \item \textbf{Slow Communication.}  Communication between clients and a central server is assumed to be the main bottleneck and dominates any computational work done at each of the clients. 
    \item \textbf{Data Privacy.} Clients want to keep their local data private, i.e., their data can not be accessed by any other client nor by the central server.
    \item \textbf{Data heterogeneity.} The training data are not independent and identically distributed (i.i.d.). In other words, a client’s local data cannot be regarded as samples drawn from single overall distribution.
    \item \textbf{Partial Participation.} Unlike traditional distributed learning systems, an FL system does not have control over individual client devices, and clients may have limited availability for connection. 
\end{itemize}
\end{assumption}

Most of the previous work in FO focus on directly solving the primal empirical-risk minimization problem \eqref{eq:primal}
\citep{mcmahan2017communication,li2018federated,yuan2021federated,karimireddy2020mime}. The broad approach taken by these FO proposals is based on requiring the clients to independently update local models that are then shared with a central server tasked with aggregating these models. 

Dual approaches for empirical-risk minimization problem \eqref{eq:primal} are well developed under the framework of distributed optimization, and can be traced to dual-decomposition~\citep{zeng2008fast,joachims1999making}, augmented Lagrangian~ \citep{jakovetic2014linear} and alternating direction method of multipliers (ADMM) \citep{boyd2011distributed,wei2012distributed}. More recent approaches include ingel-step and multi-step dual accelerated (SSDA and MSDA) methods \citep{Scaman2017OptimalAF}. These methods, however, can not be directly applied under the FO setting, because they violate some of \Cref{assum:govern}. For example, the MSDA algorithm~\citep{Scaman2017OptimalAF} enjoy optimal convergence rates, but requires full clients participation in every round, which is unrealistic under \Cref{assum:govern}. 

Our approach, in contrast, is based on the dual problem, which is a separable optimization problem with a structured linear constraint~\eqref{eq:dual}. We show that the random block coordinate descent method for problems with linear constraint proposed by \citet{necoara2017random}, is especially suitable for the FL setting. Because it is important to control the amount of local computation carried out by each client, we show how to modify this method to accommodate inexact gradient oracles. We also show how Nesterov's acceleration can be used to decrease overall complexity. As a result, we obtain convergence rates that are better than other state-of-the-art FO algorithms in certain scenarios. 

Our contributions can be summarized as follows. 
\begin{enumerate}
    \item We tackle the FO problem from a dual perspective and develop a federated dual coordinate descent (FedDCD) algorithm for FO based on the random block coordinate descent (RBCD) method proposed by \citet{necoara2017random}. We show that FedDCD fits very well to the settings of FL. 
    \item We extend the FedDCD with inexact gradient oracles and Nesterov's acceleration. The resulting convergence rates are better than the current state-of-the-art FO algorithms in certain situations.
    \item We develop a complexity lower bound for FO, the lower bound suggests that there is still a gap of $\sqrt{N}$ between the rate of accelerated FedDCD and the lower bound.
\end{enumerate}

\section{Related work}
\label{sec:relatedWork}

Distributed and parallel optimization has been extensively studied starting with the pioneering work from~\citet{Bertsekas89}. In addition to the previous mentioned ADMM, SSD and MSDA methods, popular distributed optimization algorithms include randomized gossip algorithms~\citep{BoydGPS06} and various distributed first-order methods such as the distributed gradient descent~\citep{NedicO09}, distributed dual averaging~\citep{DuchiAW12}, distributed coordinate descent~\citep{RichtarikT16}, and EXTRA~\citep{ShiLWY15}.

Federated optimization~\citep{wang2021field} is a newly emerged research subject that is closely related to centralized distributed optimization. However, most existing distributed optimization algorithms cannot be directly applied to FO because of \Cref{assum:govern}. Because FL problems usually involve a large number of total data points, most existing FO algorithms for solving \eqref{eq:primal} such as mini-batch SGD (MB-SGD)~\citep{WoodworthPS20}, FedAvg (aka. local SGD)~\citep{McMahan17}, FedProx~\citep{li2018federated}, FedDualAvg~\citep{yuan2021federated},  SCAFFOLD~\citep{pmlr-v119-karimireddy20a}, MIME~\citep{karimireddy2020mime} are variants of the SGD algorithm. Methods outside of the SGD framework are not as well developed. 

The method we develop is based on a variation of coordinate descent adapted to problems with structured linear constraints. Such algorithms have been well-studied in the context of kernel support vector machine (SVM)~\citep{lut93,platt1998sequential,libsvm}. We build on the method proposed by \citet{necoara2017random}.

\section{Problem formulation}
\label{sec:problem}

The dual problem of problem \eqref{eq:primal} is given by 
\begin{equation} \label{eq:dual}
  \minimize{y_1, \dots, y_N \in \Real^d}\enspace G(y) \coloneqq \sum_{i=1}^N f_i^*(y_i) \enspace\st\enspace \sum_{i=1}^N y_i = 0, \tag{D}
\end{equation}
where $f_i^*(y) \coloneqq \sup_{w} \langle y, w \rangle - f(w)$ is the convex conjugate function of $f_i$. Throughout this paper, we denote $y^*$ as a solution of problem \eqref{eq:dual}. To obtain this dual problem let $w^*$ denote the optimal solution to the primal problem \cref{eq:primal}. By the first-order optimality condition, we know that 
  \[0 \in \partial\left(\sum_{i = 1}^N f_i\right)(w^*),\]
  which is equivalent to 
  \[w^* \overset{\rm{(i)}}{\in} \partial\left(\sum_{i = 1}^N f_i\right)^*(0) \overset{\rm{(ii)}}{=} \partial\left(\infc_{i=1}^N f_i^*\right)(0) \overset{\rm{(iii)}}{=} \bigcap_{i=1}^N \partial f_i^*(y_i^*),\]
  where $\{y_i^*\}_{i=1}^N$ are the optimal solutions to the problem \eqref{eq:dual}, and (i), (ii) and (iii), respectively, follow from  \citet[Proposition~E.1.4.3, Proposition~E.2.3.2 and Corollary~D.4.5.5]{hiriart-urruty01}.

\subsection{Assumptions and notations}

We make the following standard assumptions on each of the primal objectives $f_i$.

\begin{assumption}[Strong convexity] \label{assum:stronglyCvx}
  There exist $\alpha > 0$ such that
  \[
    f_i(x) \geq f_i(y) + \ip{\nabla f_i(y)}{x-y} + \frac{\alpha}{2} \| x - y \|^2
  \]
  for any $x, y \in \Real^d$ and $\forall i \in [N]$. This also implies that $G$ is $(1/\alpha)$ block-wise smooth.
\end{assumption}

\begin{assumption}[Smoothness] \label{assum:smooth}
  There exist $\beta > 0$ such that
  \[
    f_i(x) \leq f_i(y)+ \ip{\nabla f_i(y)}{x-y} + \frac{\beta}{2} \| x - y \|^2
  \]
  for any $x, y \in \Real^d$ and $\forall i \in [N]$. This also implies that $G$ is $(1/\beta)$ block-wise strongly convex. 
\end{assumption}

These two assumptions are critical because they yield a one-to-one correspondence between the primal and dual spaces, and allow us to interpret each $y_i$ as a local dual representation of the global model $w$. 

Our analysis also applies to the case where the parameters $\alpha$ and $\beta$ vary for each function $f_i$, but here we assume for simplicity, these are fixed for each $f_i$. 

The data-heterogeneity between clients can be measured as the diversity of the function $f_i$, as measured by the gradient. In the convex case, it is sufficient to measure the diversity of functions only at the optimal point $w^*$; see~\citet[Assumption~3a]{koloskova2020unified}. 
\begin{assumption}[Data heterogeneity] \label{assum:dissimilar}
  Let $w^* = \arg\min~ F(w)$. There exist $\zeta>0$ such that for any $i \in [N]$,
  \[
    \left\|\nabla f_i(w^*) \right\| \leq \zeta.
  \] The relationship between $y_i^*$ and $\nabla f_i(w^*)$ implies that $\|y_i^*\| \leq \zeta ~\forall i \in [N]$, where $y^*$ is the solution of problem~\eqref{eq:dual}. 
\end{assumption}

A basic assumption in FL is that the central server does not have control over clients' devices and can not guarantee their per-round participation. Partial participation, therefore, is a respected feature for efficient FL. Here we follow the standard random participation model~\citep{wang2021field,li2020federated,li2019convergence} and assume that there is a fixed number of randomly generated clients participating in each round of the training.

\begin{assumption}[Random partial participation] \label{assum:partial}
    There exist a positive integer $\tau \in \{2,\dots,N\}$, such that in each round, only $\tau$ clients uniformly randomly distributed among the set of all clients, who can communicate with the central server.
\end{assumption}

Now we introduce some notations. For any integer $N$, we denote $[N]$ as the set $\{1,2,\ldots, N\}$. Given $I \subseteq [N]$ and $\{g_i \in \Real^d \mid i \in [N]\}$, we define the concatenation $g_I \in \Real^{Nd}$ as 
% $ g_{I}[(i-1)d + 1 : i\cdot d] = g_i $ if $i \in I$ and $ g_{I}[(i-1)d + 1 : i\cdot d] = 0 $ otherwise.
\[
    g_{I}[(i-1)d + 1 : i\cdot d] = 
    \begin{cases} 
      g_i & i \in I ;\\
      0 & \mbox{otherwise} .
    \end{cases}
\]
Given $I \subseteq [N]$, we define the linear manifold 
\[
    \Cscr_I= \left\{y \in\Real^{Nd} ~\Bigg|~ y_i\in\Real^d, ~\sum_{i\in I}y_i = 0\right\}.
\]
It follows that $\Cscr_{[N]}$ corresponds to the constraint set in \cref{eq:dual}. Let $e_d \in \Real^d$ denote the vector of all ones, and $e_I$ denote vector where $e_i = 1$ if $i \in I$ and $e_i = 0$ elsewhere.  For any positive definite matrix $W\in\Real^{Nd\times Nd}$, we define the weighted norm as $\|x\|_W := \sqrt{x^TWx}$. The projection operator onto the set $\Cscr_I$ with respect to the weighted norm $\|\cdot\|_W$ is defined as
\[
    \proj_{\Cscr_I}^{W}(x) = \argmin_{y} \|y-x\|_W^2 \enspace\st\enspace y \in \Cscr_I.
\]
% We introduce some frequently-used notation. 
% \begin{itemize}
%     \item Given $I \subseteq [N]$ and $\{g_i \in \Real^d \mid i \in [N]\}$, define the concatenation $g_I \in \Real^{Nd}$ as 
%     \[
%     g_{I}[(i-1)n + 1 : in] = 
%         \begin{cases} 
%           g_i & i \in I ;\\
%           0 & \mbox{otherwise} .
%         \end{cases}
%     \]
%     \item Given $I \subseteq [N]$, define the linear manifold 
%     \[\Cscr_I= \left\{y \in\Real^{Nd} ~\Bigg|~ y_i\in\Real^d, ~\sum_{i\in I}y_i = 0\right\}.\]
%     It follows that $\Cscr_{[N]}$ corresponds to the constraints in \eqref{eq:dual}.
%     \item Let $e_d \in \Real^d$ denote the d-vector of all ones. 
%     \item Given $w\in\Real_+^N$, we define $D_w = \diag(\pmb{w})$, where $\pmb{w}\in\Real^{Nd}$ with 
%     \[
%     \pmb{w}[(i-1)d + 1 : id] = w_i e_d
%     \]
%     for all $i \in [N]$.
%     \item The projection operator onto the set $\Cscr_I$ with respect to the weighted norm $\|\cdot\|_W$ is defined as
%     \[
%         \proj_{\Cscr_I}^{W}(x) = \argmin_{y} \|y-x\|_W^2 \enspace\st\enspace y \in \Cscr_I,
%     \]
%     where the weighted norm is defined as $\|x\|_W := \sqrt{x^TWx}$ and $W\in\Real^{Nd\times Nd}$ is any positive definite matrix. 
% \end{itemize}
  
\section{Federated dual coordinate descent}
\label{sec:FedDCD}

\citet{necoara2017random} proposed a random block coordinate descent (RBCD) method for solving linearly constrained separable convex problems such as the dual problem \eqref{eq:dual}. Below we describe how to apply this method in the FL setting, and we refer to the specialization of this algorithm as federated dual coordinate descent (FedDCD).

A training round proceeds as follows. In round $t$, suppose that the local dual representations $y_i^{(t)}$ are dual feasible, i.e., 
\[\sum_{i=1}^N y_i^{(t)} = 0.\]
First, the central server receives the IDs of participating clients $I = \{i_1, \dots, i_\tau\} \subseteq \{1,\dots,N\}$. Next, each participating client computes a local primal model $w_i^{(t)}$, which can also be interpreted as a descent direction for the dual representation, in parallel, i.e., 
\begin{equation} \label{eq:exact_gradient}
    w_i^{(t)} = \nabla f_i^*(y_i^{(t)}) =  \argmin_{w \in \Real^d}\enspace \left\{ f_i(w) - \ip{w}{y_i^{(t)}} \right\}  \text{for all} i \in I.
\end{equation}
In principle, each participating client must exactly minimizes $f_i - \ip{\cdot}{y_i^{(t)}}$, using a potentially expensive procedure, to obtain $w_i^{(t)}$. We show in the next section how the clients may instead produce an approximate primal model $w_i^{(t)}$ using a cheaper procedure.
Then each participating client sends the computed local primal model $w_i^{(t)}$ to the central server. Subsequently, the central server then adjusts the uploaded primal models to make sure that the local dual representations will still be dual feasible after getting updated. Specifically, it will compute new local primal models $\{\hat w_i^{(t)} \mid i \in I\}$ as 
\begin{equation} \label{eq:direction_adjust2}
    \hat w_I^{(t)} = \proj_{\Cscr_I}^{\Lambda}(\Lambda^{-1}w_I^{(t)}), 
\end{equation}
where $\Lambda\in\Real^{Nd \times Nd} \coloneqq \diag( \lambda_1, \ldots, \lambda_N ) \otimes \mathbb{I}_{d\times d}$ is a pre-defined diagonal matrix that usually depends on the clients' local strong convexity parameter. It can be shown that the updated directions have the closed form expressions:
\begin{equation*} %\label{eq:direction_adjust2}
    \hat w_i^{(t)} = \lambda_i^{-1}  w_i^{(t)} - \frac{\lambda_i^{-1}}{\sum_{j \in I} \lambda_j^{-1}}\sum_{j \in I} \lambda_j^{-1} w_j^{(t)}  \text{for all} i \in I. 
\end{equation*}

Finally, the central server will send back each participating client the updated primal models, who will update their local dual representations accordingly, i.e.,
\[y_i^{(t+1)} = y_i^{(t)} - \eta^{(t)}\hat w_i^{(t)},\]
where $\eta^{(t)}$ is the learning rate. \Cref{alg:rbcd} summarizes all of these steps.

\begin{algorithm}[t]
 \DontPrintSemicolon
 \SetKwComment{tcp}{\tiny [}{]}
 \caption{Federated Dual Coordinate Descent (FedDCD)\label{alg:rbcd}}
 \smallskip
 \KwIn{number of participating clients in each round $1<\tau\leq n$; diagonal scaling matrix $\Lambda\in\Real^{Nd \times Nd}$}
 $y_i^{(0)} \gets 0$ for all $i \in [N]$\tcp*{\tiny initialize feasible local dual variables}
 \For{\nllabel{alg-level-set-loop}$t\gets0,1,2,\ldots,T$}{
   $I_t \leftarrow \mbox{ set of } \tau \mbox{ random participating clients }$\tcp*{\tiny random participating clients}
   \For{each client $i\in I_t$ \textbf{in parallel}}{
        Compute local gradient $w_i^{(t)}$ as in \cref{eq:exact_gradient}\tcp*{\tiny local computation}
        Upload local gradient $w_i^{(t)}$ to the central server\tcp*{\tiny upload step} 
   }
   Compute updates $\hat w_i^{(t)}$ for all $i\in I_t$ as in \cref{eq:direction_adjust2}\tcp*{\tiny adjust directions}
   Send the updated directions to the participating clients $I_t$\tcp*{\tiny download step} 
   \For{each client $i\in [N]$ \textbf{in parallel}}{
        \uIf{$i \in I_t$}{
            $y_i^{(t+1)} \leftarrow y_i^{(t)} - \eta^{(t)}\hat w_i^{(t)}$\tcp*{\tiny update} 
        }\Else{
            $y_i^{(t+1)} \leftarrow y_i^{(t)}$\tcp*{\tiny standby}
        } 
   }
 }
 \Return{$w^{(T)} = w_i^{(T)}$}, where $i$ uniform randomly sampled from $[N]$.\tcp*{\tiny primal global model} 
\end{algorithm}

We can obtain a convergence rate for this method by directly applying results derived by \citet{necoara2017random}.

\begin{theorem}[Convergence rate of FedDCD] \label{thm:necoaraConvergenceRate}
    Let $w^{(T)}$ and $y^{(T)}$ be the iterates generated from \Cref{alg:rbcd} after $T$ iterations with the diagonal scaling matrix $\Lambda = \alpha^{-1} \mathbb{I}_{Nd \times Nd}$. If \Cref{assum:stronglyCvx} and \Cref{assum:smooth} hold, then
    \begin{equation}
        \mE \left[ G( y^{(T)} ) - G( y^* ) \right] \leq \left( 1 - \frac{\tau - 1}{N-1} \frac{\alpha}{\beta} \right)^{T} ( G(y^{(0)}) - G(y^*) ), \label{eq:necoaraDualConvergenceStrongCVX}
    \end{equation}
    and
    \begin{equation}
        \mE \left[ \| w^{(T)} - w^* \|^2 \right] 
        \leq \frac{1}{N \alpha^2} \left( 1 - \frac{\tau - 1}{N-1} \frac{\alpha}{\beta} \right)^{T} \|y^*\|^2  \label{eq:necoaraPrimalConvergenceStrongCVX}.
    \end{equation}
    If in addition that \cref{assum:dissimilar} holds, then
    \begin{equation}
        \mE \left[ \| w^{(T)} - w^* \|^2 \right] 
        \leq \frac{1}{\alpha^2} \left( 1 - \frac{\tau - 1}{N-1} \frac{\alpha}{\beta} \right)^{T} \zeta^2. \label{eq:necoaraPrimalConvergenceDissimilar}
    \end{equation}
\end{theorem}

The rate \eqref{eq:necoaraDualConvergenceStrongCVX} is based on a minor modification of~\citet[Theorem~3.1]{necoara2017random}. By strong convexity of the primal problem, we are able to extend the result to primal variables, which is shown in \cref{eq:necoaraPrimalConvergenceStrongCVX}. Furthermore, if we assume a bound on the data heterogeneity, i.e. \Cref{assum:dissimilar}, then we can get a better convergence rate on the primal variables; see \cref{eq:necoaraPrimalConvergenceDissimilar}.

\subsection{Applicability under federated learning}
In this section, we discuss the properties of FedDCD (\Cref{alg:rbcd}) and argue that it is a suitable algorithm for FL in the sense that it respects the governing FL assumptions (\cref{assum:govern}), which are different from classical distributed optimization, as described in \cref{sec:intro}. Specifically, 
\begin{itemize}
    \item \textbf{Reduced communication.} 
    %\Cref{thm:necoaraConvergenceRate} shows that in order to achieve $\epsilon$ accuracy in primal variable, i.e., $\|w^{(T)} - w^*\|^2 \leq \epsilon$, we need $\Oscr(\log(\frac{1}{\epsilon}))$ rounds of communication.
    \Cref{tab:communicationComplexity} summaries the communication complexities of some existing FO methods such as mini-batch SGD (MB-SGD)~\citep{WoodworthPS20}, FedAvg (local SGD)~\citep{mcmahan2017communication}, SCAFFOLD~\citep{pmlr-v119-karimireddy20a}. Compared to the rates of other algorithms, our communication complexity only involves a logarithmic dependence on $\epsilon$ (but with a cost that $N$ appears in the nominator). These rates implies that when $N/\tau$ is not too large, FedDCD converges faster than other algorithms.
    On the other hand, if $N/\tau$ is large, i.e. the participation rate is small, MB-SGD and FedAvg (local SGD) converges faster since their convergence rates are independent from the number of clients. 
    \item \textbf{Data privacy.} As shown in \Cref{alg:rbcd}, our method only requires clients to send local model updates to the server, which is similar to most existing FO methods \citep{mcmahan2017communication,li2018federated,yuan2021federated,karimireddy2020mime}. Local data privacy is thus preserved. 
    \item \textbf{Data heterogeneity.} 
    The data heterogeneity between clients is captured by the parameter $\zeta$ in \Cref{assum:dissimilar}, i.e., a larger $\zeta$ indicates greater data heterogeneity between clients. \Cref{eq:necoaraPrimalConvergenceDissimilar} reveals the impact of $\zeta$ on the convergence rate. In the extreme case when all the clients have same local data, i.e., $f_i = f_j$ for all $i,j \in [N]$, \Cref{alg:rbcd} will reach the optimal point at the first iteration.
    \item \textbf{Partial participation.} By design, \cref{alg:rbcd} only needs $\tau$ clients to participate in each round, where $\tau$ can be any number between $2$ and $N$. This feature offers flexibility for the numebr of participating clients in each round. \Cref{thm:necoaraConvergenceRate} also implies that the convergence rate improves as more clients participate. Note that the convergence analysis can also reveal the behaviour of the method when the number of participating clients is allowed to vary across rounds.
\end{itemize}

\begin{table}[t]
\small
\centering
     \begin{tabular}{lc}    
     \toprule 
     algorithm & rounds  \\
     \midrule
     MB-SGD & $\displaystyle \BigOh \left( \frac{\sigma^2}{\tau \alpha \epsilon}  + \left( 1 - \frac{\tau}{N} \right) \frac{\zeta^2}{\tau \alpha \epsilon} + \frac{\beta}{\alpha} \log\left(\frac{1}{\epsilon}\right) \right) $ \\ 
     FedAvg (local SGD) & $\displaystyle \BigOh \left( \frac{\sigma^2}{\tau K \alpha \epsilon} + \frac{\sqrt{\beta}( \zeta K + \sigma \sqrt{K} ) }{\alpha \sqrt{\epsilon}} + \frac{\beta}{\alpha} \log\left(\frac{1}{\epsilon}\right) \right)$  \\
     SCAFFOLD & $\displaystyle \BigOh \left( \frac{\sigma^2}{\tau K \alpha \epsilon}  + \left(\frac{\beta}{\alpha} +\frac{N}{\tau} \right)\log\left(\frac{1}{\epsilon}\right) \right) $  \\
     FedDCD (Ours)& $\displaystyle \BigOh\left( \frac{(N-1) \beta}{(\tau-1) \alpha} \log\left( \frac{\zeta}{ \epsilon } \right) \right)$ \\
     AccFedDCD (Ours)& $\displaystyle \BigOh\left( \frac{(N-1) }{(\tau-1) } \sqrt{ \frac{\beta}{\alpha} } \log\left( \frac{\zeta}{ \epsilon } \right) \right)$ \\
     \bottomrule 
     \end{tabular}
    %  \captionsetup{justification=centering}
     \caption{Communication complexities of different algorithms, where $K$ is the number of local steps of FedAvg and SCAFFOLD, $\sigma$ is the gradient variance bound. To make a fair comparison, we present our rate in averaged empirical risk instead of the sum of empirical risk (we divide \cref{eq:primal} by $N$). See \citet{WoodworthPS20} for a more comprehensive summary of the rates. The rate of AccFedDCD is from \Cref{sec:accFedDCD}.\label{tab:communicationComplexity}}
\end{table}

% Here we present a convergence rate of \Cref{alg:rbcd} under a weaker condition.
% \begin{theorem}[Convergence under weaker conditions]
%     Assume that \Cref{assum:boundedDomain} is satisfied. Given $T \in \mathbb{N}$, let $\{ y^{(t)} \}_{t=1}^T$ be the iterates generated from \Cref{alg:rbcd} with the diagonal scaling matrix $\Lambda = \diag( \gamma_1^{-1}, \gamma_2^{-1}, \ldots, \gamma_N^{-1} ) \otimes \mathbb{I}_{d \times d}$ and $\eta_t = \sqrt{ \frac{(N-1)(\sum_{i=1}^N \gamma_i \| y_i^{0} - y_i^* \|^2)}{2(\tau-1)(\sum_{i=1}^N \gamma_i)T} }$. Then 
%     \[
%         \min_{t \in [T]} \mE \left[ F( y^{(t)} ) - F( y^* ) \right] \leq \sqrt{ \frac{ (N-1) ( \sum_{i=1}^N \gamma_i ) ( \sum_{i=1}^N \gamma_i \| y_i^{(0)} - y_i^* \|^2 ) }{(\tau-1)T} }.
%     \]
% \end{theorem}

\section{Inexact federated dual coordinate descent}
\label{sec:inexactFedDCD}
A drawback of \Cref{alg:rbcd} is that the calculation of local primal model $w_i^{(t)}$ requires exact minimization of the individual primal objective, e.g., solving \eqref{eq:exact_gradient}. This could potentially be computationally prohibitive as clients may have limited local computational resources. 
A natural remedy is to solve \ref{eq:exact_gradient} inexactly. The convergence of coordinate descent with inexact gradients has been studied by~\citet{CassioliLS13,TappendenRG16,LiuSY21}. \citet{LiuSY21} recently extended the MSDA algorithm~\citep{Scaman2017OptimalAF} by using \emph{lazy dual gradients}. Our approach is related to their work and we use some of their intermediate results to build our analysis. 

%We introduce an inexact version of FedDCD as a remedy with little loss in convergence rate. 
First, we introduce the inexact gradient oracle. 
\begin{definition}[$\delta$-inexact gradient oracle] 
Given a function $u : \Real^p \to \Real$ and $\delta \in (0,1)$, we say that $\texttt{oracle}_{u, \delta}(x, g_{\init})$ is a $\delta$-inexact gradient oracle for $u$ if it outputs $\wt{\nabla} u(x)$ as an approximation of $\nabla u(x)$ that satisfies
\begin{align}
        \mE\left[\| \wt{\nabla} u(x) - {\nabla} u(x) \|^2\right] \leq \delta \| g_{\init} - {\nabla} u(x) \|^2, \label{eq:inexactOracle1}
\end{align}
where $g_{\init}$ is an initial guess of the true gradient $\nabla u(x)$ and the expectation is taken over the oracle itself as the oracle can be a randomized procedure.
\end{definition}
To incorporate inexact gradient oracle into FedDCD, we only need to modify line 5 of \Cref{alg:rbcd} to
\begin{equation} \label{eq:inexact_gradient}
    w_i^{(t)} = \texttt{oracle}_{f_i^*, \delta}(y^{(t)}_i, w_{i}^{(t-1)}),
\end{equation}
where $\texttt{oracle}_{f_i^*, \delta}$ is a $\delta$-inexact gradient oracle for $f_i^*$ and 
we let $w_i^{(-1)} \coloneqq 0~\forall i \in [N]$.

Next, we show that the $\delta$-inexact gradient oracle can be implemented by running some standard optimization algorithms locally on clients' devices. Specifically, when \Cref{assum:smooth} and \Cref{assum:stronglyCvx} are satisfied, we can bound the number of local updates required to satisfy \cref{eq:inexact_gradient} with different algorithms. Assume that client $i$ has $n_i$ data points on its device, we list some standard algorithms for solving \cref{eq:inexact_gradient} below along with the number of local steps required:
\begin{itemize}
    \item Gradient descent with initial point $w_i^{(t-1)}$ needs $\BigOh\left( \frac{ n_i \beta}{\alpha} \log\left( \frac{1}{\delta} \right) \right)$ gradient steps for $i \in [N], t \in \mathbb{N}$ \citep{cvxopt_lecture};
    \item When each individual loss functions $f_i$'s are finite sums such that $f_i = \sum_{j=1}^{n_1} f_{i,j}~\forall i \in [N]$, and all their inner losses $f_{i,j}$'s are $\alpha$-strongly convex and $\beta$-smooth, variance-reduced SGD such as SAG, SAGA and SVRG with initial point $w^{(t-1)}_i$ needs $\BigOh\left( \left( \frac{\beta}{\alpha} + n_i \right) \log\left( \frac{1}{\delta} \right) \right)$ stochastic gradient steps for $i \in [N], t \in \mathbb{N}$ \citep{roux2012stochastic,JohnsonZ13,DefazioBL14};
    \item Inexact Newton method with initial point $w^{(t-1)}_i$ needs $\BigOh( \log\left( \frac{1}{\delta} \right) )$ inexact Newton steps for $i \in [N], t \in \mathbb{N}$~\citep{dembo1982inexact}.
\end{itemize}

% All these methods show that for any fixed $\delta$, the clients can implement the inexact gradient oracle with a fixed number of local computational steps. 
When $\delta$ is fixed as a constant, all the above methods require only a fixed number of local training steps in each round, and the number of local steps is independent from the number of the final accuracy $\epsilon$. 
The inexact gradient oracle offers more flexibility than other federated learning algorithms as they usually require all client to use specific algorithm to perform local updates, whereas inexact FedDCD allows clients to choose their own local solver that are suitable to the computational power of their local device. This mechanism can potentially mitigate the device heterogeneity issue in FL.

We then show the convergence rate of FedDCD with inexact gradient oracle.
\begin{theorem}[Convergence rate of inexact FedDCD] \label{thm:inexactConvergence}
    Suppose that \Cref{assum:stronglyCvx} and \Cref{assum:smooth} hold. Define an auxiliary constant
    \[\kappa ~=~ \frac{ (\tau-1) \alpha}{32 (N-1) \beta}.\]
    Let $w^{(T)}$ and $y^{(T)}$ be the iterates generated from \Cref{alg:rbcd} with $\delta$-inexact gradient oracle where $\delta=(1-\kappa)/4$, learning rate $\eta^{(t)} = 1/4$ for all $t$, and diagonal scaling matrix $\Lambda = \alpha^{-1} \mathbb{I}_{Nd \times Nd}$. Then
    \begin{equation} \label{eq:InexactDualConvergenceStrongCVX}
        \mE \left[ G( y^{(T)} ) - G( y^* ) \right] \leq \left( 1 - \kappa \right)^{T} ( G(y^{(0)}) - G(y^*)  ), 
    \end{equation}
    and
    \begin{equation} \label{eq:InexactPrimalConvergenceStrongCVX}
        \mE \left[ \| w^{(T)} - w^* \|^2 \right] \leq \frac{20}{3N\alpha^2}(1-\kappa)^T  \|y^*\|^2. 
    \end{equation}
    If in addition that \Cref{assum:dissimilar} holds, then
    \[\mE \left[ \| w^{(T)} - w^* \|^2 \right] \leq \frac{20}{3\alpha^2}(1-\kappa)^T \zeta^2.\]
\end{theorem}

\Cref{thm:inexactConvergence} implies that \Cref{alg:rbcd} can still enjoy linear convergence rate when using appropriate gradient accuracy $\delta$ and learning rate $\eta^{(t)}$ as specified above. We only suffer from a loss in the constant term compared with the convergence rate with exact gradient oracles, c.f., \Cref{thm:necoaraConvergenceRate}.

\section{Accelerated federated dual coordinate descent}
\label{sec:accFedDCD}

In this section, we apply Nesterov's acceleration to \Cref{alg:rbcd} and obtain improved convergence rates. Random coordinate descent with Nesterov's acceleration has been widely studied; see~\citet{Nesterov12,LeeS13,LinLX15,ZhuQRY16,NesterovAccCD}. However, the analysis in the literature is almost exclusively focused on unconstrained problems or problems with separable regularizers and therefore does not apply to \cref{eq:dual} because of the linear constraint. In the following of this section, we adapt the accelerated randomized coordinate descent algorithm to problems linear constraints.

\begin{algorithm}[t]
 \DontPrintSemicolon
 \SetKwComment{tcp}{\tiny [}{]}
 \caption{Accelerated FedDCD \label{alg:fedArbcd}}
 \smallskip
 \KwIn{number of selected clients in each round $1<\tau\leq n$; diagonal scaling matrix $\Lambda\in\Real^{Nd \times Nd}$}
 $y_i^{(0)}, z_i^{(0)}, v_i^{(0)} \gets 0$ for all $i \in [N]$, let $r =\frac{\tau-1}{N-1}$, $a = \frac{\sqrt{\alpha/\beta}}{\frac{1}{r} + \sqrt{\alpha/\beta}}, b = \frac{\alpha a r^2}{\beta}$\tcp*{\tiny initialize feasible local dual variables}
 \For{\nllabel{alg-level-set-loop}$t\gets0,1,2,\ldots$}{
    \For{each client $i \in [N]$ \textbf{in parallel}}{
   $v_i^{(t)} = (1-a) y_i^{(t)} + a z_i^{(t)}$  \tcp*{\tiny update}
   }
   $I_t^1 \leftarrow \mbox{ set of } \tau \mbox{ random participating clients }$\tcp*{\tiny random clients}
   \For{each client $i\in I_t^1$ \textbf{in parallel}}{
        Compute local gradient $w_i^{(t)}$ as in \cref{eq:exact_gradient} with $y_i^{(t)}$ replaced by $v_i^{(t)}$\tcp*{\tiny local computation}
        Upload local gradient $w_i^{(t)}$ to the central server\tcp*{\tiny upload step} 
   }
   Compute updates $\hat w_i^{(t)}$ for all $i\in I_t^1$ as in \cref{eq:direction_adjust2}\tcp*{\tiny adjust directions}
   Send the updated directions to the participating clients $I_t^1$\tcp*{\tiny download step} 
   \For{each client $i\in [N]$ \textbf{in parallel}}{
        $u_i^{(t)} = \frac{a^2}{a^2 + b} z_i^{(t)} + \frac{b}{a^2 + b} v_i^{(t)} $  \tcp*{\tiny update}
        \uIf{$i \in I_t^1$}{
            $y_i^{(t+1)} \leftarrow v_i^{(t)} - \hat w_i^{(t)}$\tcp*{\tiny update} 
        }\Else{
            $y_i^{(t+1)} \leftarrow v_i^{(t)}$\tcp*{\tiny standby}
        } 
   }
   $I_t^2 \leftarrow \mbox{ set of } \tau \mbox{ random participating clients }$\tcp*{\tiny random clients}
   \For{each client $i\in I_t^2$ \textbf{in parallel}}{
        Compute local gradient $w_i^{(t)}$ as in \cref{eq:exact_gradient}  with $y_i^{(t)}$ replaced by $v_i^{(t)}$\tcp*{\tiny local computation}
        Upload local gradient $w_i^{(t)}$ to the central server\tcp*{\tiny upload step} 
   }
   Compute updates $\hat w_i^{(t)}$ for all $i\in I_t^2$ as in \cref{eq:direction_adjust2}\tcp*{\tiny adjust directions}
   Send the updated directions to the participating clients $I_t^2$\tcp*{\tiny download step} 
   \For{each client $i\in [N]$ \textbf{in parallel}}{
        \uIf{$i \in I_t^2$}{
            $z_i^{(t+1)} \leftarrow u_i^{(t)} - \frac{ar}{a^2 + b} \hat w_i^{(t)}$\tcp*{\tiny update} 
        }\Else{
            $z_i^{(t+1)} \leftarrow u_i^{(t)}$\tcp*{\tiny standby}
        } 
   }
 }
 \Return{$w^{(T)} = w_i^{(T)}$}, where $i$ uniform randomly sampled from $[N]$.\tcp*{\tiny primal global model} 
\end{algorithm}

The accelerated FedDCD method is detailed in \Cref{alg:fedArbcd}. 
The detailed algorithm is shown in , we call it accelerated FedDCD. Accelerated FedDCD follows the standard algorithm template of accelerated coordinate descent: we introduce auxiliary variables $v_i^{(t)}, u_i^{(t)}$ and $z_i^{(t)}$, where $v_i^{(t)}$ is a linear combination of $y_i^{(t)}$ and $z_i^{(t)}$ (see line 18 of \Cref{alg:fedArbcd}) and $u_i^{(t)}$ is a linear combination of $z_i^{(t)}$ and $v_i^{(t)}$ (see line 26 of \Cref{alg:fedArbcd}). In each round, we sample two sets of clients and calculate their gradients, the calculation of gradient is based on the variable $v_i^{(t)}$ instead of $y_i^{(t)}$ (see line 21 and line 33).
The major difference between accelerated FedDCD and standard accelerated RCD is in line 23 and 35, where accelerated FedDCD requires a partial projection step to keep the updated iterates feasible.
Note that the standard efficient implementation of accelerated RCD requires a  change of variable technique \citep{LeeS13}, which is not necessary in our setting because all clients can independently update their local models in parallel.

The convergence rate of accelerated FedDCD is given by the following result, which shows that it enjoys essentially the same convergence rate of standard accelerated RCD for unconstrained problems~\citep{LeeS13,Lu18}.

\begin{theorem}[Convergence rate of accelerated FedDCD] \label{thm:accFedDCDStronglyCVX}
     Let $y^{(T)}$ and $w^{(T)}$ be the iterates generated from \Cref{alg:fedArbcd} with diagonal scaling matrix $\Lambda = \alpha^{-1} \mathbb{I}_{Nd \times Nd}$. Suppose that \Cref{assum:stronglyCvx} and \Cref{assum:smooth} hold, then
    \begin{equation} \label{eq:necoaraDualConvergenceStrongCVXAcc}
        \mE \left[ G( y^{(T)} ) - G( y^* ) \right] ~\leq~ \left( 1 - \frac{ \sqrt{\frac{\alpha}{\beta}} }{ \frac{N-1}{\tau -1}+ \sqrt{\frac{\alpha}{\beta}} } \right)^{T} ( G(y^{(0)}) - G(y^*) ),
    \end{equation}
    and
    \begin{equation} \label{eq:necoaraPrimalConvergenceStrongCVXAcc}
        \mE \left[ \| w^{(T)} - w^* \|^2 \right] 
        \leq \frac{1}{N\alpha^2} \left( 1 - \frac{ \sqrt{\frac{\alpha}{\beta}} }{ \frac{N-1}{\tau -1}+ \sqrt{\frac{\alpha}{\beta}} } \right)^{T} \|y^*\|^2.
    \end{equation}
    If in addition that \cref{assum:dissimilar} holds, then
    \begin{equation} \label{eq:necoaraPrimalConvergenceDissimilarAcc}
        \mE \left[ \| w^{(T)} - w^* \|^2 \right] 
        \leq \frac{1}{\alpha^2} \left( 1 - \frac{ \sqrt{\frac{\alpha}{\beta}} }{ \frac{N-1}{\tau -1}+ \sqrt{\frac{\alpha}{\beta}} } \right)^{T} \zeta^2. 
    \end{equation}
\end{theorem}

\Cref{thm:accFedDCDStronglyCVX} implies that the iteration complexity (and thus the bound on the number of communication rounds) of accelerated FedDCD is 
\begin{equation}
        \BigOh\left( \frac{(N-1) }{(\tau-1) } \sqrt{ \frac{\beta}{\alpha} } \log\left( \frac{ \zeta}{ \epsilon } \right) \right) , \label{eq:accFedDCDcomplexity}
\end{equation}
which improved the condition number $\frac{\beta}{\alpha}$ found in \eqref{eq:necoaraPrimalConvergenceDissimilar} to $\sqrt{\frac{\beta}{\alpha}}$. In the appendix, we also show that when all the functions $f_i$ are only strongly convex but not necessarily smooth (the dual is smooth but not strongly convex), then accelerated random coordinate descent with linear constraint converges with rate $\BigOh( \epsilon^{-\frac{1}{2}} )$ on the dual problem.

We believe that the accelerated version of FedDCD can also use inexact gradient oracle without compromising its convergence rate; we leave this for future work.

\section{Complexity lower bound under random participation}
\label{sec:lowerBound}

We follow the black-box procedure from \cite{Scaman2017OptimalAF,Scaman2018OptimalAF} and propose the following constraints for the black-box optimization procedures of FO under random participation:
\begin{enumerate}
    \item \textbf{Clients' memory:} at time $t$, each client $i$ can store the past models, denoted by $\Mscr_{i,t} \subset \Real^d$. The stored models  either come from each client's local update or client-server communication, that is
    \[
        \Mscr_{i,t} = \Mscr_{i,t}^{comp} \cup \Mscr_{i,t}^{comm} \quad \forall i \in [N].
    \]
    \item \textbf{Clients' local computation:} at time $t \geq 0$, the clients can update their local model via arbitrary first-order oracles for arbitrary steps:
    \[
        \Mscr_{i,t}^{comp} = \bigcup_{k=0}^{\infty} \Ascr_k,
    \]
    where
    \begin{align*}
        \Ascr_0 &= \mathrm{Span}\left( \left\{  w, \nabla f_i(w), \nabla f_i^*(w) : w \in \Mscr_{i,t-1}  \right\} \right), \\
        \Ascr_k &= \mathrm{Span}\left( \left\{  w, \nabla f_i(w), \nabla f_i^*(w) : w \in \Ascr_{k-1}  \right\} \right) \qquad \forall k \geq 1.
    \end{align*}
    \item \textbf{Client-server communication:} at time $t \geq 0$, the server can collect the models from the randomly generated participating clients at time $t-1$:
    \[
        \displaystyle \Mscr_{server, t} = \mathrm{Span} \left( \Mscr_{server, t-1} \cup  \left( \bigcup_{i \in \Nscr_t} \Mscr_{i, t-1} \right) \right),
    \]
    where the set of participating clients $\Nscr_t$ is uniformly generated from $\{1, \dots, N\}$ and $|\Nscr_t| = \tau \in \{2, \dots, N\}$. The client could also receive a new model through the communication with the central server:
    \[
        \Mscr_{i,t}^{comm} = \left\{ w \right\}, \qquad w \in \Mscr_{server, t-1},
    \]
    where $w$ could an arbitrary model from $\Mscr_{server, t-1}$.
    \item \textbf{Output model:} at time $t$, the server selects one model in its memory as output:
    \[
        w^{(t)} \in \Mscr_{server, t}.
    \]
\end{enumerate}

For simplicity, we assume all clients and server have trivial initialization $M_{server, 0}, M_{i,0} = \{0\}~\forall i \in [N]$. We also assume that in every round the communication cost is 1 time unit, and that the server and clients are allowed to conduct local computation within each time unit. The major difference between our black-box procedure and the distributed optimization black-box procedure from \citet{Scaman2017OptimalAF} is that we have an additional constraint on random participation that constrains the server from communicating with no more than $\tau$ uniform randomly-generated clients in each round.

\begin{theorem}[Lower bound] \label{thm:lowerBound}
    There exist $N$ functions $f_i$'s that satisfy \Cref{assum:stronglyCvx} and \Cref{assum:smooth}, such that for any positive integer $T$ and any algorithm from our black-box procedure we have
    \begin{align}
       \mE_{ \Nscr_i, i \in [T] } \left[ f( w^{(T)} ) - f(w^*) \right]  ~\geq~ \Omega \left( \left( 1 - \min\left\{ \frac{\tau}{\sqrt{N}} \frac{4}{ \sqrt{ \beta/\alpha }}, \frac{8 \sqrt{2} }{\beta/\alpha} , 1\right\} \right)^T \right)  \label{eq:lowerBound}.
    \end{align}
\end{theorem}

\Cref{thm:lowerBound} implies that the iteration complexity of the random participation first-order black-box procedure is bounded below by
\[
    \Omega \left( \frac{\sqrt{N}}{\tau} \sqrt{ \frac{\beta}{\alpha} } \ln\left( \frac{1}{\epsilon} \right) \right)
\]
when $\beta/\alpha$ is large. Compared to the iteration complexity of accelerated FedDCD (\cref{eq:accFedDCDcomplexity}), there is a gap $\sqrt{N}$ between the lower and upper bound, which suggests that the rate of accelerated FedDCD may not be optimal. It is an open problem whether the lower bound can be further tightened or if there is a algorithm with better rate than accelerated FedDCD.

\section{Experiments}
\label{sec:experiments}

We conduct experiments on real-world datasets to evaludate the effectiveness of FedDCD and accelereated FedDCD. We include the following algorithms for comparison. 
\begin{itemize}
    \item \textbf{Primal methods} (baseline algorithms): FedAvg~\citep{mcmahan2017communication}, FedProx~\citep{li2018federated} and SCAFFOLD~\citep{pmlr-v119-karimireddy20a}. 
    \item \textbf{Dual methods}: FedDCD (\Cref{alg:rbcd}) and AccFedDCD (\Cref{alg:fedArbcd}).
\end{itemize}

\begin{figure}[t]
\begin{subfigure}{.33\textwidth}
  \centering
  \includegraphics[width=\linewidth]{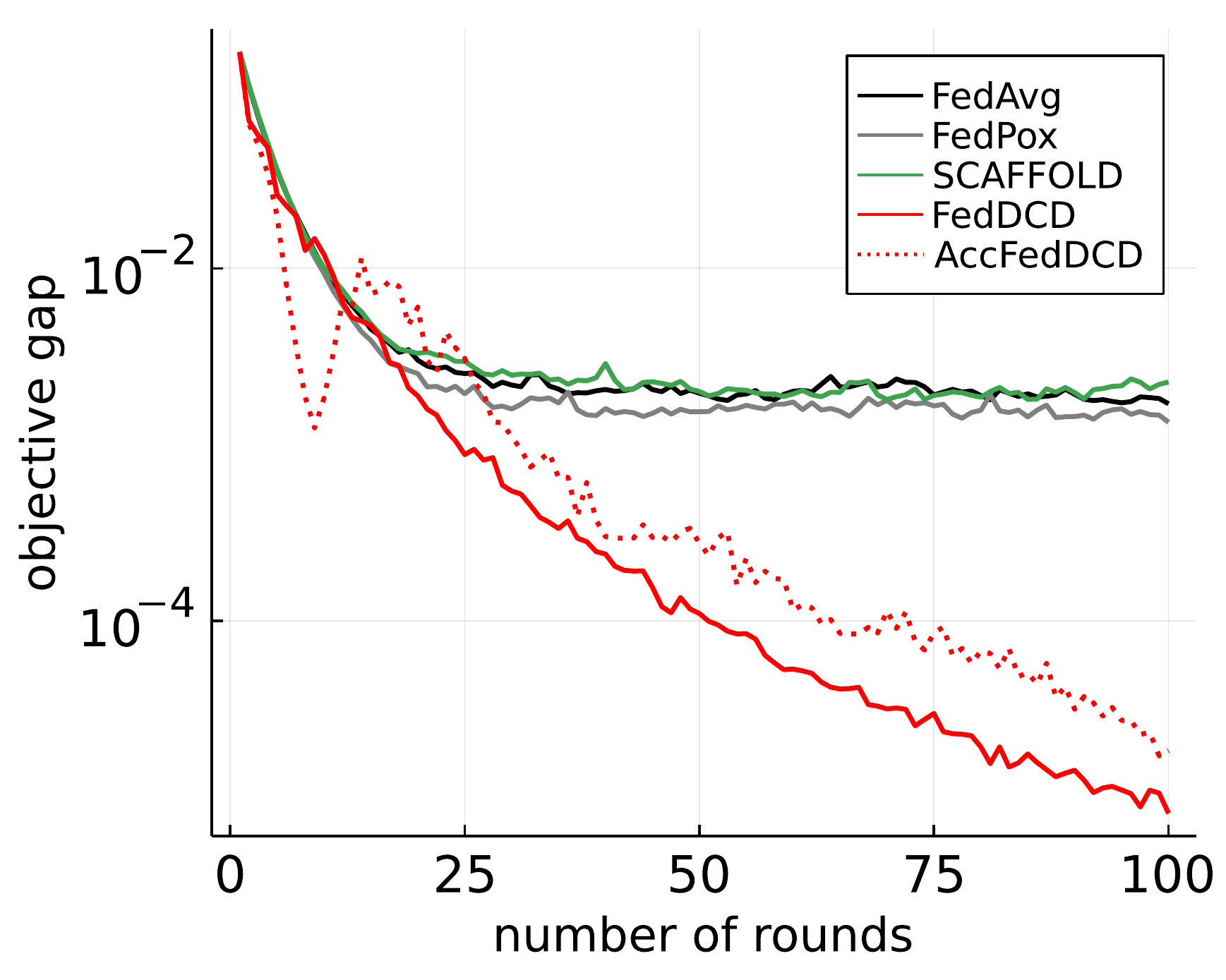}
  \caption{RCV1 with MLR model.}
  \label{fig:exp1_rcv1_mlr}
\end{subfigure}%
\begin{subfigure}{.33\textwidth}
  \centering
  \includegraphics[width=\linewidth]{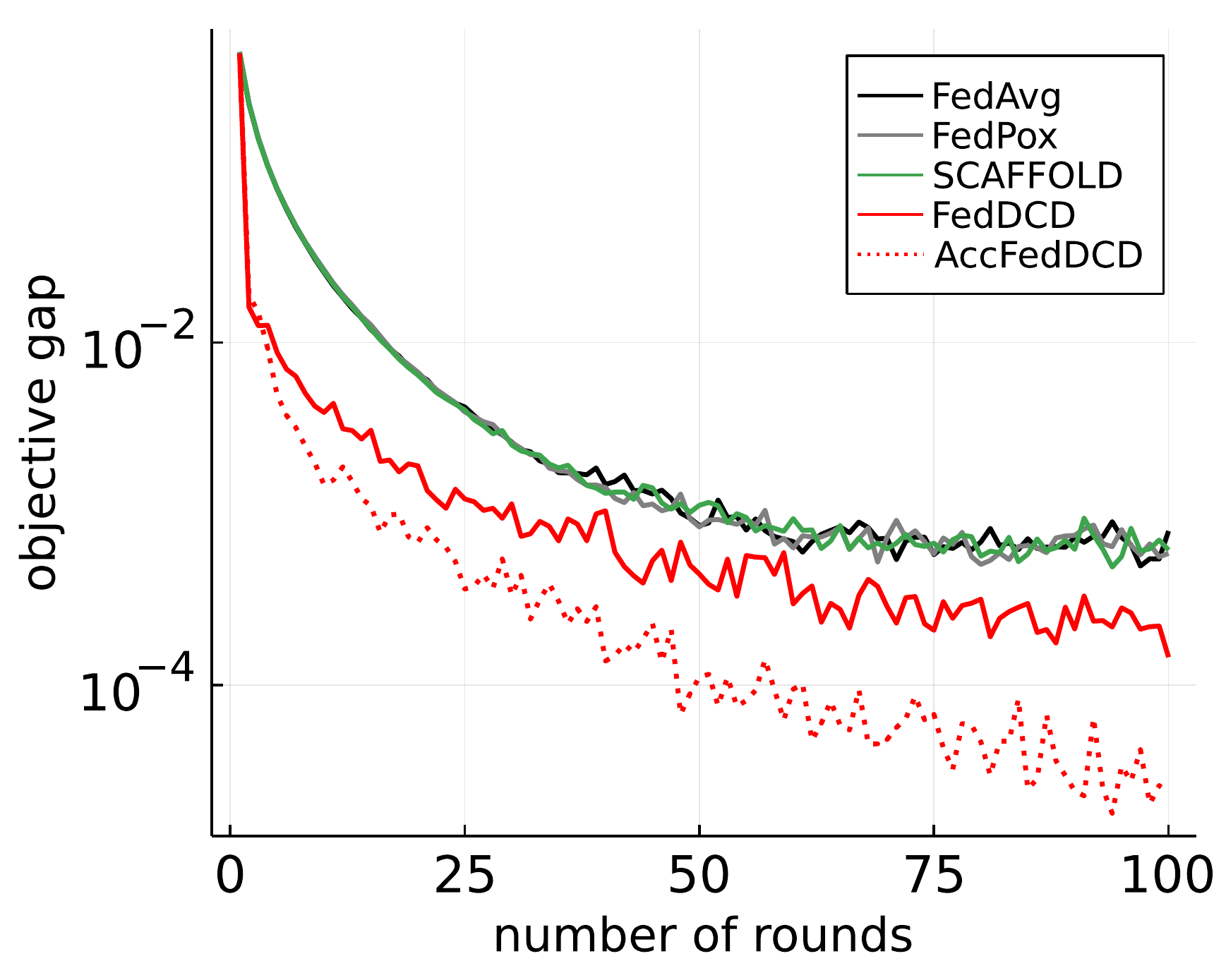}
  \caption{MNIST with MLR model.}
  \label{fig:exp1_mnist_mlr}
\end{subfigure}
\begin{subfigure}{.33\textwidth}
  \centering
  \includegraphics[width=\linewidth]{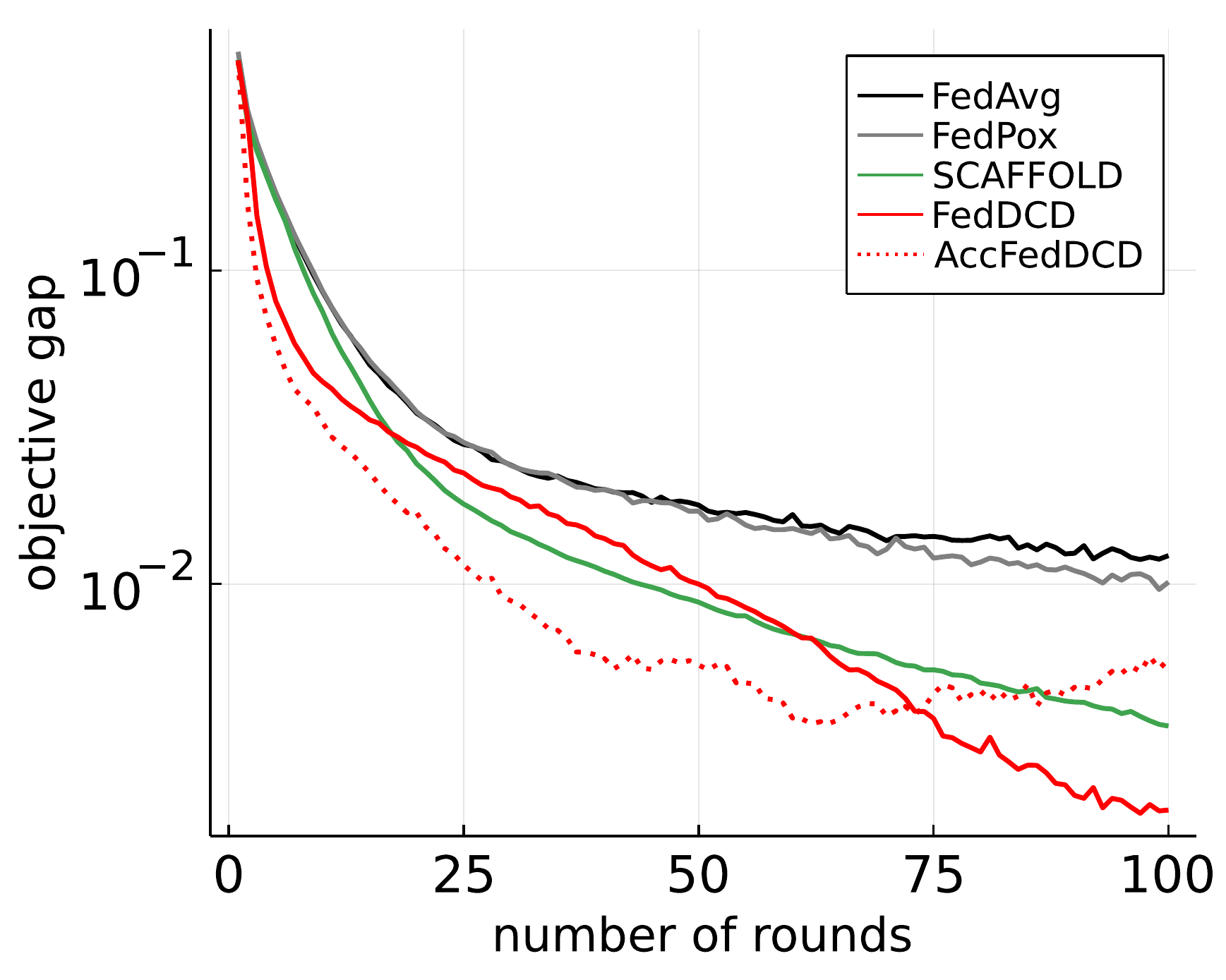}
  \caption{MNIST with MLP model.}
  \label{fig:exp1_mnist_mlp}
\end{subfigure}
\caption{Comparison between primal and dual methods; see~\cref{sec:exp1}} 
\label{fig:exp1}
\end{figure}

\subsection{Data sets and implementation} \label{sec:details}
\paragraph{RCV1} The first dataset we use is the Reuters Corpus Volume I (RCV1) dataset~\citep{lewis2004rcv1}, where the task is to categorize newswire
stories provided by Reuters, Ltd. for research purposes. The number of training points is $20,242$, the number of test points is $677,399$, the number of features is $47,236$ and the number of classes is $2$.

\paragraph{MNIST} The second dataset we use is the well known MNIST dataset~\citep{lecun1998gradient}, where the task is to classify hand-written digits. The number of training points is $60,000$, the number of test points is $10,000$, the number of features is $784$ and the number of classes is $10$.

All datasets are downloaded from the website of LIBSVM\footnote{\url{https://www.csie.ntu.edu.tw/~cjlin/libsvmtools/datasets/}}~\citep{libsvm}.

\paragraph{Model} For RCV1 dataset, we train a multinomial logistic regression (MLR) model. For MNIST dataset, we train two models: a MLR model and a 2-layer multilayer perceptron (MLP) model with 32 neurons in the hidden layer.

\paragraph{Data distribution} For the experiments in~\cref{sec:exp1} and~\cref{sec:exp2}, we distribute the data to clients in an i.i.d. fashion, i.e., the local datasets are uniformly sampled without replacement and the numbers of local training samples are equal among all the clients. For the experiments in~\cref{sec:exp3},  we distribute the data to clients in a non-i.i.d. fashion, i.e., each client gets samples of only two classes and the numbers of local training samples are not equal, which is the same setting as in the FedAvg paper~\citep{mcmahan2017communication}.

\paragraph{Implementation}  We set the number of clients to be $100$ for all experiments. We implement the variants of FedDCD algorithm proposed in this paper, as well as the primal methods mentioned above in the Julia language~\citep{bezanson2017julia}. Our code is publicly available at \url{https://github.com/ZhenanFanUBC/FedDCD.jl}. For the primal methods, we try the number of local epochs to be $5$ or $20$ and report the best result. For dual methods, when the model is MLR, we compute the local gradient via $10$ steps of Newton updates, and when the model is MLP, we compute the local gradient via $20$ steps of ADAM~\citep{kingma2014adam} updates (we locally perform 5 epochs SGD when the data is non-i.i.d. distributed). All the experiments are conducted on a Linux server with 8 CPUs and 64 GB memory. 

\subsection{Comparison between primal and dual methods} \label{sec:exp1}
We compare the performances between the well-known primal methods listed above and the dual method we proposed. We set the number of active clients in each round $\tau = 30$ for all methods. The experiment results is shown in~\cref{fig:exp1}, 
%where the x-axis represents the number of communication rounds, and the y-axis represent the gap between the current objective and optimal objective. 
As we can see from the plots, FedDCD and AccFedDCD have better convergence in terms of communication for the MLR models compared with primal methods and perform similarly as SCAFFOLD for the MLP model.

\subsection{Impact of participation rate} \label{sec:exp2}
We examine the impact of participation rate for both primal and dual methods. We set the number of active clients in each round $\tau \in \{ 5, 10, 30\}$ for all the primal and dual methods. We report the number of communication rounds required by different algorithms to achieve certain objective gap $\epsilon \in \{ 10^{-3}, 10^{-2}, 10^{-1}\}$. The results are summarized in \Cref{tab:impactOfParticipationRate}. We observe that FedDCD and AccFedDCD outperform other primal methods in most settings, the trend is obvious especially when the participation rate is high and the target objective gap is small.
The observation is consistent with our analysis in \Cref{sec:FedDCD}.

\begin{table}[t]
    \centering
    \begin{tabular}{ccccc}
    \toprule
    & & MLR for MNIST& {MLR for RCV1} & {MLP for MNIST}  \\ \hline
    Setup & Algorithm & \# rounds  & \# rounds  & \# rounds \\  \cline{1-5}
    \multirow{5}{*}{$\tau=30, \epsilon=10^{-3}$} & FedAvg & 51 &  $>100$  & $>100$  \\ 
    & FedProx & 47  & $>100$  & $>100$ \\
    & SCAFFOLD & 45  & $>100$  & $>100$ \\
    & FedDCD & 28  & \textbf{24}  & $>100$ \\
    & AccFedDCD & \textbf{15} & 30  & $>100$ \\ \cline{2-5}
    \multirow{5}{*}{$\tau=30, \epsilon=10^{-2}$} & FedAvg & 17 &  10  & $>100$  \\ 
    & FedProx & 17  & 10  & 99 \\
    & SCAFFOLD & 17  & 11  & 44 \\
    & FedDCD & 4  & 10  & 49 \\
    & AccFedDCD & \textbf{3} & \textbf{5}  & \textbf{28} \\ \cline{2-5}
    \multirow{5}{*}{$\tau=10, \epsilon=10^{-2}$} & FedAvg & 19 &  \textbf{11}  & $>500$  \\ 
    & FedProx & 18  & \textbf{11}  & $>500$ \\
    & SCAFFOLD & 18  & \textbf{11}  & $418$ \\
    & FedDCD & 19  & 27  & \textbf{96} \\
    & AccFedDCD & \textbf{14} & 14  & $>500$ \\ \cline{2-5}
    \multirow{5}{*}{$\tau=5, \epsilon=10^{-1}$} & FedAvg & 5 &  3  & 88  \\ 
    & FedProx & 5  & 3  & 83 \\
    & SCAFFOLD & 5  & 3  & 79 \\
    & FedDCD & 3  & \textbf{5}  & 5 \\
    & AccFedDCD & \textbf{1} & \textbf{1}  & \textbf{4} 
    \\ \bottomrule
    \end{tabular}
    \caption{Number of rounds required under different setups, $\epsilon$ stands for the target objective gap. We bold the smallest number of rounds for each setup of $\tau$ and $\epsilon$.}  \label{tab:impactOfParticipationRate}
 \end{table}

\subsection{Impact of data heterogeneity} \label{sec:exp3}
We examine the impact of data heterogeneity for both primal and dual methods. We distribute the data to clients in an non-i.i.d. fashion as described in \cref{sec:details}. We set the number of active clients in each round $\tau = 30$ for all the primal and dual methods. The training curves are shown in \cref{fig:exp3} and the final testing accuracies from different algorithms are reported in \Cref{tab:exp3}.
%We show the experiment results in~\cref{fig:exp3} and~\Cref{tab:exp3}. In~\cref{fig:exp3}, the x-axis represents the number of communication rounds, and the y-axis represent the gap between the current objective and optimal objective. In~\Cref{tab:exp3}, we show the test accuracies for all the algorithms. 
As we can see from the plots, for the MLR model, AccFedDCD out performs all the other algorithms, and for the MLP model, FedDCD and AccFedDCD have similar performance as SCAFFOLD. Besides, when compared with Figure~\ref{fig:exp1_mnist_mlr} and Figure~\ref{fig:exp1_mnist_mlp}, the convergence behaviors of FedDCD and AccFedDCD become worse, which reflects the impact of data heterogeneity. From \Cref{tab:exp3}, we can observe that both FedDCD, AccFedDCD can reach testing accuracy as good as SCAFFOLD.

\begin{table}[t]
    \begin{minipage}[c]{0.65\textwidth}%
    \centering
        \begin{subfigure}{.5\textwidth}
          \centering
          \includegraphics[width=.95\linewidth]{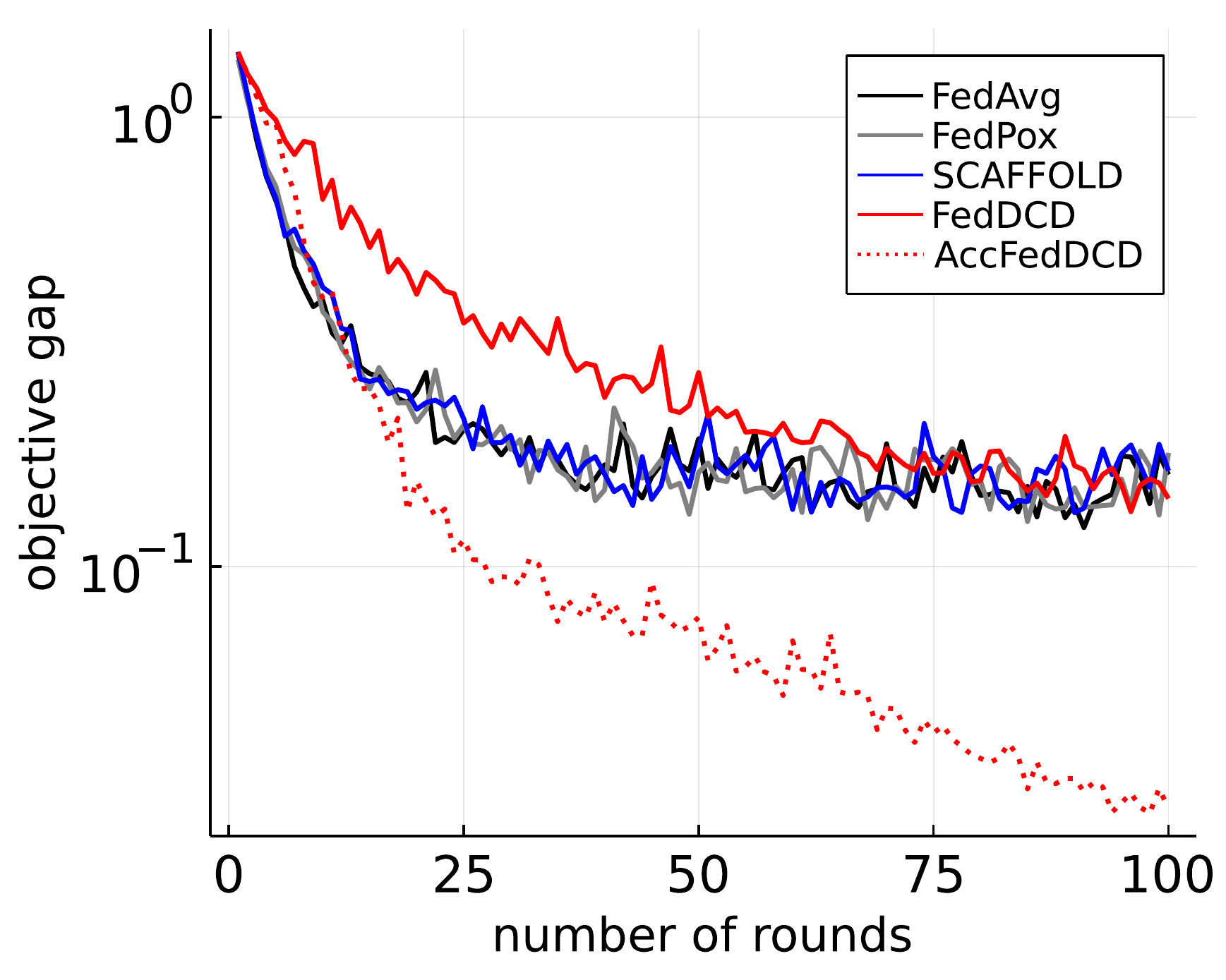}
          \caption{MNIST with MLR model.}
          \label{fig:exp3_mnist_mlr}
        \end{subfigure}%
        \begin{subfigure}{.5\textwidth}
          \centering
          \includegraphics[width=.95\linewidth]{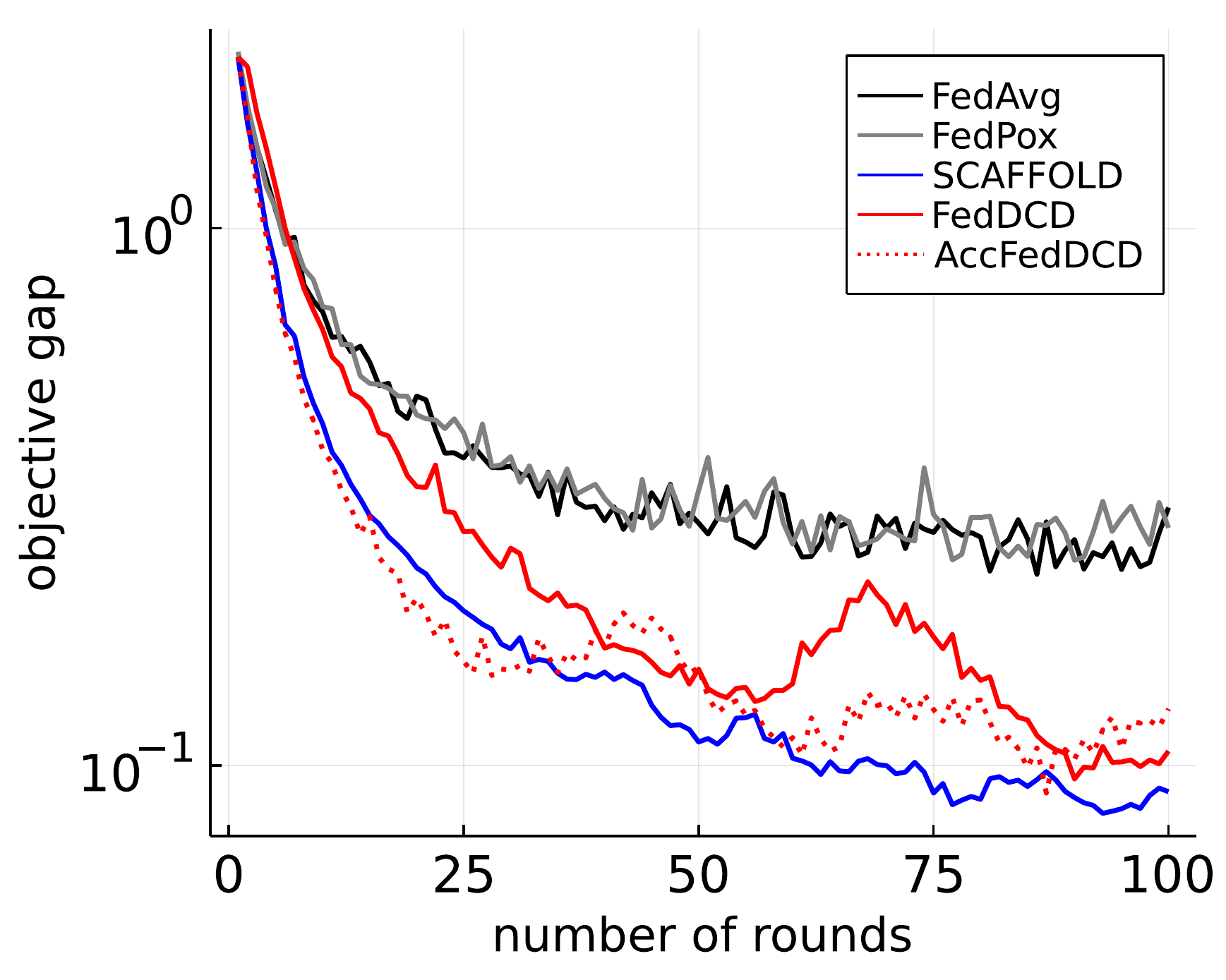}
          \caption{MNIST with MLP model.}
          \label{fig:exp3_mnist_mlp}
        \end{subfigure}
        \captionof{figure}{Impact of data heterogeneity; see~\cref{sec:exp3}.} 
        \label{fig:exp3}
    \end{minipage}
    \hfil
    \begin{minipage}[c]{0.3\textwidth}%
    \small
    \centering
        \begin{tabular}{ccc}
        \toprule
                        &  MLR      & MLP \\ \hline
            FedAvg      &  88.30\%  & 88.17\% \\
            FedProx     &  88.67\%  & 87.92\% \\ 
            SCAFFOLD    &  88.47\%  & 90.23\% \\
            FedDCD      &  88.16\%  & 90.02\% \\
            AccFedDCD   &  89.24\%  & 90.32\% \\
        \bottomrule
        \end{tabular}
        \caption{Test accuracy.}
        \label{tab:exp3}
    \end{minipage}
\end{table}

% \begin{figure}[t]
% \begin{subfigure}{.5\textwidth}
%   \centering
%   \includegraphics[width=.65\linewidth]{figures/exp3_mnist_mlr.pdf}
%   \caption{MNIST with MLR model.}
%   \label{fig:exp3_mnist_mlr}
% \end{subfigure}%
% \begin{subfigure}{.5\textwidth}
%   \centering
%   \includegraphics[width=.65\linewidth]{figures/exp3_mnist_mlp.pdf}
%   \caption{MNIST with MLP model.}
%   \label{fig:exp3_mnist_mlp}
% \end{subfigure}
% \caption{Impact of data heterogeneity; see~\cref{sec:exp3}} 
% \label{fig:exp3}
% \end{figure}

% \begin{table}[t]
%     \centering
%     \begin{tabular}{ccc}
%     \toprule
%                     &  MLR      & MLP \\ \hline
%         FedAvg      &  88.30\%  & 88.17\% \\
%         FedProx     &  88.67\%  & 87.92\% \\ 
%         SCAFFOLD    &  88.47\%  & 90.23\% \\
%         FedDCD      &  88.16\%  & 90.02\% \\
%         AccFedDCD   &  89.24\%  & 90.32\% \\
%     \bottomrule
%     \end{tabular}
%     \caption{Test accuracy.}
%     \label{tab:exp3}
% \end{table}

\section{Conclusion and future directions}
\label{sec:conclusion}

In this paper, by tackling the dual problem of federated optimization, we propose the federated dual coordinate descent (FedDCD) algorithm and its variants based on the random block coordinate descent algorithm~\citep{necoara2017random}. Both FedDCD and its variants satisfy the desired properties for federated learning and have better communication complexities than other SGD-based federated learning algorithms under certain scenarios. 

More importantly, FedDCD provides a general framework for federated optimization and suggests many interesting future research directions. First, it is possible to develop an asynchronous version of FedDCD by leveraging well-studied analysis of asynchronous parallel coordinate descent methods~\citep{LiuW15,0002WRBS15}. Next, one might consider client sampling strategies other than the standard uniform sampling. For example, there are some recent studies of the coordinate descent with the greedy selection rule~\citep{NutiniSLFK15,BCD_julie,fang2020greed}, which can be adopted with FedDCD. Finally, the lower bound of complexity of first-order methods with random participation for federated optimization is still an open problem. As we have shown in \Cref{sec:lowerBound}, our lower bound analysis has a $\BigOh(\sqrt{N})$ gap to the upper bound of accelerated FedDCD. We hope to explore whether the lower bound can be further tightened or an algorithm with a faster convergence rate can be developed. 

% Some possible future directions:
% \begin{itemize}
%     \item Asynchronous version. 
%     \item Greedy version.
%     \item Different participation rate.
%     \item Optimal convergence rate.  
% \end{itemize}

\bibliographystyle{plainnat}
\bibliography{refs/shorttitles, refs/master, refs/friedlander}

\newpage
\appendix

\section{Structure of the appendix}

In this appendix, we include some materials to supplement the main context. To make our argument cleaner, we consider the following optimization problem 
\begin{align}
    \minimize{x_1, \dots, x_n \in \Real}\enspace h(x) \coloneqq \sum_{i=1}^n h_i(x_i) \enspace\st\enspace \sum_{i=1}^n x_i = 0 \label{eq:CDObj}.
\end{align}
For simplicity, we assume $h_i$'s to be 1-dimensional scalar functions and all the results presented in the appendix can be easily extended to the block case. In this appendix, we analyze the convergence behavior of several extensions to the randomized block coordinate descent (RBCD) method proposed by~\citet{necoara2017random}. In \Cref{appendix:lemmas}, we present some useful lemmas. We provide a convergence analysis for inexact RBCD and accelerated RBCD respectively in \Cref{appendix:inexact_rbcd} and \Cref{appendix:accelerated_rbcd}. We provide the proofs for all the theorems in the main context in \Cref{appendix:derivation_main_context}. Finally, in \Cref{appendix:lower_bound}, we show the complexity lower bound for solving problem \eqref{eq:primal}. 

% We present some useful lemmas in  and provide the following extensions over~\citet{necoara2017random}'s block randomized coordinate descent (BRCD) with linear constraint:
% \begin{itemize}
%     \item We provide a convergence analysis of BRCD with linear constraint for convex and Lipschitz continuous objective (but not necessarily smooth).
%     \item We give an accelerated BRCD algorithm with linear constraint, the resulting algorithm could obtain convergence rate $\BigOh(1/T^2)$.
%     \item We derive convergence rate when using inexact gradient information.
% \end{itemize}
% Note that applying CD to solve problem in form~\cref{eq:CDObj} is not new. For example, the dual SVM with a bias term can also be written as~\cref{eq:CDObj}, and there is a line of research on the analysis of CD for~\cref{eq:CDObj}~\cite{tsy10}. The previous analysis of CD applied to~\cref{eq:CDObj} assumed the objective to be smooth. To our knowledge, our extensions described above has not appeared in the literature.

\section{Preliminaries and Lemmas} \label{appendix:lemmas}
In this section, we introduce some assumptions and notations used in the appendix. 
Throughout the appendix, we assume $h$ to be coordinate-wise smooth and strongly convex, which is formalized in the following assumption. 
\begin{assumption}[Structure of function $h$] \label{ass:hStruncture} 
  There exists positive constants $\{L_i \mid i \in [N]\}$ and $\{\mu_i \mid i \in [N]\}$ such that for all $x,y \in \Real$ and $i \in [N]$, 
  \begin{align}
      h_i(x) &\leq h_i(y) + \langle \nabla h_i(y), x - y \rangle + \frac{L_i}{2} (x - y)^2 \tag{smoothnes}
      \\ h_i(x) &\geq h_i(y) + \langle \nabla h_i(y), x - y \rangle + \frac{\mu_i}{2} (x - y)^2 \tag{strong convexity}.
  \end{align}
    
\end{assumption}

Let $L_{\max} \coloneqq \max_{i \in [n]} L_i, L_{\min} \coloneqq \min_{i \in [n]} L_i, \mu_{\max} \coloneqq \max_{i \in [n]} \mu_i, \mu_{\min} \coloneqq \min_{i \in [n]} \mu_i$ and $\bL = \diag(L_1, L_2, \ldots, L_n)$. Sometimes we will simply write $\mu_{\min}$ as $\mu$ in our analysis. 
We denote the set of coordinates selected as $I \subseteq [n]$. Given a vector $x \in \Real^n$, we define $x_I \coloneqq \sum_{i \in I} x_i e_i $. The identity matrix is written as $\mathbb{I}$ and $\mathbb{I}_I$ is the diagonal matrix with $i$th diagonal element equal to $1$ if $i \in I$ and $0$ otherwise. The constraint sets
\[
    \Cscr = \left\{x\in\Real^n ~\Bigg|~ \sum_{i=1}^n x_i=0\right\} \quad \text{and} \quad \Cscr_I=\left\{x\in\Real^n ~\Bigg|~ \sum_{i\in I}x_i = 0\right\}
\]
are used repeatedly in our analysis. Given $1 < \tau \leq n$, we use $\mathcal{P}_\tau$ to denote all possible subsets of $[n]$ that has cardinality $\tau$, e.g., $\mathcal{P}_\tau = \{ I \subseteq [n] \mid |I| = \tau \}$. We consider the distribution
\[
    \mathbb{P}(I) = \frac{e_I^T\bL^{-1}e_I}{\sum_{J \in \Pscr_\tau} e_J^T\bL^{-1}e_J}
\]
when generating random indices set $I \subseteq [n]$. This distribution is identical to the one used in~\citet{necoara2017random}. We refer readers to their work for more intuition behind this sampling scheme. Given a PSD matrix $W\in \Real^{n \times n}$, we define the norm $\|x\|_W^2 \coloneqq x^T W x~\forall x \in \Real^{n}$ and we also define $\langle x,y \rangle_{W}$ as $x^T W y$. The projection operator on $\Cscr_I$ with respect to $\|\cdot\|_W$ is defined as
\[
    \proj_{\Cscr_I}^{W}(x) = \argmin_{y} \|y-x\|_W^2 \enspace\st\enspace y \in \Cscr_I.
\]
We define the following four operators, which are widely used in our analysis:
\[
    P_I \coloneqq \proj_{\Cscr_I}^{\bL} \circ \mathbb{I}_I, \quad  G_I \coloneqq P_I \circ \bL^{-1}, \quad P_\tau \coloneqq \mE_I [P_I], \quad G_\tau \coloneqq \mE_I [G_I].
\]
Next we present some useful lemmas for our analysis.
% \begin{itemize}
%     \item $I \subseteq [n]$;
%     \item $W = \diag(w_1, \dots, w_n)$ with $w_i > 0$ for all $i$;
%     \item $\mathbb{I} \in \Real^{n\times n}$ is the identity matrix;
%     \item $\mathbb{I}_I \in \Real^{n\times n}$ is the diagonal matrix with $i$th diagonal element equal to $1$ if $i \in I$ and $0$ otherwise;
%     \item $\Cscr = \left\{d\in\Real^n \mid \sum_{i=1}^n d_i=0\right\}$;
%     \item $\Cscr_I=\left\{d\in\Real^n\mid \sum_{i\in I}d_i = 0\right\}$;
%     \item $\Pscr_\tau = \{I \subseteq [n] \mid |I| = \tau\}$;
%     \item the probability of $I$ being selected is $\mathbb{P}(I) = \frac{e_I^TW^{-1}e_I}{\sum_{J \in \Pscr_\tau} e_J^TW^{-1}e_J}$;
%     \item $P_I = \proj_{\Cscr_I}^{W} \circ \mathbb{I}_I$ and $G_I = P_I \circ W^{-1}$. 
% \end{itemize}

\begin{lemma}[Expression of the projection operator] \label{lemma:projection1}
The projection operator on the set $C_I$ can be expressed as 
    \[\proj_{\Cscr_I}^{W}(x) = \left(I - \frac{1}{e_I^T W^{-1} e_I}W^{-1}e_I e_I^T\right)x.\]
\end{lemma}
\begin{proof}
    By definition, the projection operator can be expressed as 
\begin{equation} \label{eq:proj_cI}
    \proj_{\Cscr_I}^{W}(x) = \argmin_{d} \frac{1}{2} (d - x)^TW(d-x) \enspace\st\enspace e_I^Td = 0.
\end{equation}
The Lagrangian function with respect to \cref{eq:proj_cI} takes the form 
\[\Lscr(d, \lambda) = \frac{1}{2} (d - x)^TW(d-x) + \lambda e_I^Td.\]
By checking the optimility condition with respect to $d$, we have
\[\nabla_d \Lscr(d, \lambda) = W(d - x) + \lambda e_I = 0 \enspace\Rightarrow\enspace d = x - \lambda W^{-1} e_I.\]
By checking the optimility condition with respect to $\lambda$, we have 
\[
    \nabla_\lambda \Lscr(d, \lambda) = e_I^Td
    = e_I^Tx - \lambda e_I^T W^{-1} e_I = 0 
    \enspace\Rightarrow\enspace
    \lambda = \frac{e_I^T x}{e_I^T W^{-1} e_I}.
\]
We can thus conclude that 
\[\proj_{\Cscr_I}^{W}(x) = \left(I - \frac{1}{e_I^T W^{-1} e_I}W^{-1}e_I e_I^T\right)x.\]
\end{proof}

As a consequence of \cref{lemma:projection1}, we have the following expressions for $P_I$ and $G_I$:
\begin{align*}
    P_I &= \mathbb{I}_I - \frac{1}{e_I^T\bL ^{-1}e_I}\bL^{-1}e_Ie_I^T\\
    G_I &= \mathbb{I}_I\bL^{-1} - \frac{1}{e_I^T\bL^{-1}e_I}\bL^{-1}e_Ie_I^T\bL^{-1}.
\end{align*}
Our next lemma gives explicit expressions for the expectations of $P_I$ and $G_I$. 
\begin{lemma}[Statistical properties of matrices $P$ and $G$] \label{lemma:matrix_expectation}
    For any $I \subseteq [n]$, we have the following relationships:
    \begin{align}
        G_{I} &= G_{I}^T, \tag{Symmetric} \\
        P_{I} &= P_{I}^2, \tag{Idempotent} \\
        G_{I} &= G_{I}^T \bL G_{I}, \label{eq:GeqGTLG} \\
        P_\tau \coloneqq \mathbb{E}_I[P_I] &= \frac{\tau - 1}{n - 1} P_{[n]}, \label{eq:Pnexp} \\
        G_\tau \coloneqq \mathbb{E}_I[G_I] &= \frac{\tau - 1}{n - 1} G_{[n]}, \label{eq:Gnexp} \\
        \mE_I \left[\| G_I x \|^2_{\bL} \right] &= \frac{\tau-1}{n-1}  \| G_{[n]} x \|^2_{\bL} \qquad \forall x \in \Real^n, \label{eq:GnormExp}\\
        \mE_I \left[ \ip{G_I x}{y}_{\bL}\right] &= \frac{\tau - 1}{n - 1}\ip{x}{y} \qquad \forall x,y \in \Cscr. \label{eq:GIpExp}
    \end{align}
\end{lemma}
\begin{proof}
    The symmetry of $G_{I}$ and the idempotent of $P_{I}$ follow directly from the definitions. 
    
    For \cref{eq:GeqGTLG}, we have
    \begin{align*}
        G_{I}^T \bL G_{I} &~=~ G_{I} \bL G_{I} ~=~ P_{I} \bL^{-1} \bL P_{I} \bL^{-1} ~=~ P_{I} \bL^{-1} ~=~ G_{I}.
    \end{align*}

    The expression of $G_\tau$ follows directly from \cite[Theorem~3.3]{necoara2017random}. So we only need to derive the expression fro $P_\tau$. By definition, we have 
    \begin{align*}
        P_\tau &= \mathbb{E}_I[P_I]\\
        &= \sum_{I \in \Pscr_\tau} \mathbb{P}(I) \left[ \mathbb{I}_I - \frac{1}{e_I^TW^{-1}e_I}W^{-1}e_Ie_I^T \right]\\
        &= \sum_{I \in \Pscr_\tau} \frac{e_I^TW^{-1}e_I}{\sum_{J \in \Pscr_\tau} e_J^TW^{-1}e_J} \left[ \mathbb{I}_I - \frac{1}{e_I^TW^{-1}e_I}W^{-1}e_Ie_I^T \right]\\
        &= \underbrace{\sum_{I \in \Pscr_\tau} \frac{e_I^TW^{-1}e_I}{\sum_{J \in \Pscr_\tau} e_J^TW^{-1}e_J} \mathbb{I}_I}_{(a)} - \underbrace{\sum_{I \in \Pscr_\tau} \frac{1}{\sum_{J \in \Pscr_\tau} e_J^TW^{-1}e_J}W^{-1}e_Ie_I^T}_{(b)}.
    \end{align*}
    Let $\Sigma_\tau = \sum_{J \in \Pscr_\tau} e_J^TW^{-1}e_J$. Then the first component can be expressed as 
    \begin{align*}
        (a) &= \Sigma_\tau^{-1} \sum_{I \in \Pscr_\tau} e_I^TW^{-1}e_I \mathbb{I}_I\\
        &= \Sigma_\tau^{-1} \sum_{j=1}^{\binom{n}{\tau}} \left(\sum_{u \in I_j}w_u^{-1}\right)\left(\sum_{v \in I_j}e_ve_v^T\right)\\
        &= \Sigma_\tau^{-1} \sum_{j=1}^{\binom{n}{\tau}} \sum_{u \in I_j} \sum_{v \in I_j} w_u^{-1} e_ve_v^T \\
        &= \Sigma_\tau^{-1} \sum_{u=1}^n \sum_{v=1}^n w_u^{-1} e_ve_v^T \left(\sum_{j=1}^{\binom{n}{\tau}} \mathbbm{1}_{u,v \in I_j} \right)\\
        &= \binom{n-2}{\tau-2}\Sigma_\tau^{-1}e^TW^{-1}e \mathbb{I},
    \end{align*}
    and the second component can be expressed as 
    \begin{align*}
        (b) &= \Sigma_\tau^{-1} \sum_{I \in \Pscr_\tau} W^{-1}e_Ie_I^T\\
        &= \Sigma_\tau^{-1} W^{-1} \sum_{j=1}^{\binom{n}{\tau}} \sum_{u \in I_j} \sum_{v \in I_j} e_ue_v^T\\
        &= \Sigma_\tau^{-1} W^{-1} \sum_{u=1}^n \sum_{v=1}^n e_ue_v^T \left(\sum_{j=1}^{\binom{n}{\tau}} \mathbbm{1}_{u,v \in I_j} \right)\\
        &= \binom{n-2}{\tau-2}\Sigma_\tau^{-1}W^{-1}ee^T.
    \end{align*}
    Therefore, we have 
    \[P_\tau = \binom{n-2}{\tau-2}\Sigma_\tau^{-1}e^TW^{-1}e P_{[n]}.\]
    Next, we show that $\binom{n-2}{\tau-2}\Sigma_\tau^{-1}e^TW^{-1}e = \frac{\tau-1}{n-1}$. Indeed, we have 
    \begin{align*}
        \binom{n-2}{\tau-2}\Sigma_\tau^{-1}e^TW^{-1}e &= \binom{n-2}{\tau-2}\frac{1}{\sum_{J \in \Pscr_\tau} e_J^TW^{-1}e_J}e^TW^{-1}e\\
        &= \binom{n-2}{\tau-2} \frac{1}{\binom{n-1}{\tau-1}e^TW^{-1}e}e^TW^{-1}e\\
        &= \frac{\tau-1}{n-1}.
    \end{align*}
    This finishes the proof for~\cref{eq:Pnexp}.
    
    For any $x\in \Real^n$,
    \begin{align*}
        \mE_I \left[ \| G_I x \|_{\bL}^2 \right] &~=~ \mE_I \left[ x^T G_I^T \bL G_I x \right] \\
        &~=~ \mE_I \left[ x^T G_I x \right] \qquad \text{(By \cref{eq:GeqGTLG})} \\
        &~=~ \frac{\tau - 1}{n-1}   x^T G_{[n]} x  \qquad \text{(By \cref{eq:Gnexp})} \\
        &~=~ \frac{\tau - 1}{n-1} \| G_{[n]} x \|^2_{\bL}.
    \end{align*}
    This finishes the proof for~\cref{eq:GnormExp}.
    
    Finally, we prove \cref{eq:GIpExp},
    \begin{align*}
        \mE \left[\ip{G_Ix}{y}_{\bL} \right] &= \frac{\tau - 1}{n-1} \ip{G_{[n]}x}{y}_{\bL}\\
        &= \frac{\tau - 1}{n-1} \ip{x}{G_{[n]}\bL y}\\
        &= \frac{\tau - 1}{n-1} \ip{x}{P_{[n]}\bL^{-1}\bL y}\\
        &= \frac{\tau - 1}{n-1} \ip{x}{y}  \qquad \text{(Since $y \in \Cscr$)}.
    \end{align*}
\end{proof}

\begin{lemma}[Eigenvalues of $G_{[n]}$] \label{lemma:eigenG}
    \begin{align}
        \lambda_1 ( G_{[n]} ) &\leq 1/L_{\min}, \label{eq:largestEigenValG} \\
        \lambda_n ( G_{[n]} ) & = 0. \label{eq:smallestEigenValG}
    \end{align}
\end{lemma}

\begin{proof}
    Follows directly from the definition of $G_{[n]}$. 
\end{proof}

\begin{lemma}[Gradient at optimal] \label{lemma:gradient_optimal}
    Let $x^*$ be the optimal solution to \cref{eq:CDObj}. Then we have $P_{[n]}^T\nabla h(x^*) = 0$. 
\end{lemma}

\begin{proof}
    Since $x^*$ is the optimal solution to \cref{eq:CDObj}, by the first order optimality condition, we know that $-\nabla h(x^*) \in \Nscr_\Cscr(x^*)$, where $\Nscr_\Cscr(x^*)$ is the normal cone of $\Cscr$ at $x^*$. By the definition of normal cone, we know that 
    \begin{align*}
                    ~&~ \ip{-\nabla h(x^*)}{z - x^*} \leq 0 \enspace\forall z \in \Cscr \\
        \Rightarrow ~&~ \ip{-\nabla h(x^*)}{P_{[n]} z - x^*} \leq 0 \enspace\forall z \in \Real^n \\
        \Rightarrow ~&~ \ip{-\nabla h(x^*)}{P_{[n]} (z - x^*)} \leq 0 \enspace\forall z \in \Real^n \\
        \Rightarrow ~&~ \ip{-P_{[n]}^T \nabla h(x^*)}{z - x^*} \leq 0 \enspace\forall z \in \Real^n \\
        \Rightarrow ~&~ \ip{-P_{[n]}^T \nabla h(x^*)}{z} \leq 0 \enspace\forall z \in \Real^n \\
        \Rightarrow ~&~ P_{[n]}^T \nabla h(x^*) = 0.
    \end{align*}
\end{proof}
\begin{lemma}[Bound on gradient] \label{lemma:GnormLowerAndUpperBound}
    Under \Cref{ass:hStruncture}, we have 
    \[\frac{2\mu}{L_{\max}^2}(h(x^t) - h(x^*)) \leq \|G_{[n]}\nabla h(x^t)\|^2 \leq \frac{2L_{\max}}{L_{\min}^2}(h(x^t) - h(x^*)),\]
    and
    \[\frac{2\mu}{L_{\max}}(h(x^t) - h(x^*)) \leq \|G_{[n]}\nabla h(x^t)\|_{\bL}^2 \leq \frac{2L_{\max}}{L_{\min}}(h(x^t) - h(x^*))\]
    for any $x^t \in \Cscr$.
\end{lemma}

\begin{proof}
    We define a helper function $\phi:\Real^n\to\Real$ as 
    \[\phi(x) \coloneqq h(P_{[n]}x).\]
    By the assumption that $h$ is $\mu$-strongly convex, it follows that 
    \begin{align*}
        &~ h(y) \geq h(x) + \ip{\nabla h(x)}{y-x} + \frac{\mu}{2}\|x-y\|^2 \enspace \forall x, y\\
        \Rightarrow &~  h(P_{[n]}y) \geq h(P_{[n]}x) + \ip{\nabla h(P_{[n]}x)}{P_{[n]}y-P_{[n]}x} + \frac{\mu}{2}\|P_{[n]}x-P_{[n]}y\|^2 \enspace \forall x, y\\
        \Rightarrow &~  h(P_{[n]}y) \geq h(P_{[n]}x) + \ip{P_{[n]}^T\nabla h(P_{[n]}x)}{y-x} + \frac{\mu}{2}\|P_{[n]}x-P_{[n]}y\|^2 \enspace \forall x, y\\
        \Rightarrow &~  \phi(y) \geq \phi(x) + \ip{\nabla \phi(x)}{y-x} + \frac{\mu}{2}\|P_{[n]}x-P_{[n]}y\|^2 \enspace \forall x, y\\
        \Rightarrow &~  \phi(y) \geq \phi(x) + \ip{\nabla \phi(x)}{y-x} + \frac{\mu}{2}\|x-y\|^2 \enspace \forall x, y \in \Cscr.
    \end{align*}
    Fix $x = x^t$ and take minimization with respect to $y \in \Cscr$ to both sides of the inequality, we can get 
    \[\frac{1}{2}\|\nabla \phi(x^t)\|^2 \geq \mu(\phi(x^t) - \phi(x^*)).\]
    It follows that 
    \begin{align*}
        \|G_{[n]}\nabla h(x^t)\|^2 &= \|\bL^{-1}\nabla \phi(x^t)\|^2
        \\&\geq \frac{1}{L_{\max}^2} \|\nabla \phi(x^t)\|^2
        \\&\geq \frac{2\mu}{L_{\max}^2}(\phi(x^t) - \phi(x^*))
        \\&= \frac{2\mu}{L_{\max}^2}(h(x^t) - h(x^*)).
    \end{align*}
    Similarly, for the $L$-norm, we have
    \begin{align*}
        \|G_{[n]}\nabla h(x^t)\|_{\bL}^2 &= \|\bL^{-1}\nabla \phi(x^t)\|_{\bL}^2
        \\&= \|\bL^{-\frac{1}{2}}\nabla \phi(x^t)\|^2
        \\&\geq \frac{1}{L_{\max}} \|\nabla \phi(x^t)\|^2
        \\&\geq \frac{2\mu}{L_{\max}}(\phi(x^t) - \phi(x^*))
        \\&= \frac{2\mu}{L_{\max}}(h(x^t) - h(x^*)).
    \end{align*}
    By the same reason and the assumption that $h$ is $L$-smooth, we can conclude that for any $x^t$,
    \[\phi(x^t) - \phi(x^*) \geq \frac{1}{2L_{\max}}\|\nabla \phi(x^t)\|^2.\]
    It follows that 
    \begin{align*}
        \|G_{[n]}\nabla h(x^t)\|^2 &= \|\bL^{-1}\nabla \phi(x^t)\|^2
        \\&\leq \frac{1}{L_{\min}^2}\|\nabla \phi(x^t)\|^2
        \\&\leq \frac{2L_{\max}}{L_{\min}^2} (\phi(x^t) - \phi(x^*))
        \\&= \frac{2L_{\max}}{L_{\min}^2}(h(x^t) - h(x^*)).
    \end{align*}
    Similarly, for the $L$-norm, we have
    \begin{align*}
        \|G_{[n]}\nabla h(x^t)\|_{\bL}^2 &= \|L^{-1}\nabla \phi(x^t)\|_{\bL}^2
        \\&= \|\bL^{-\frac{1}{2}}\nabla \phi(x^t)\|^2
        \\&\leq \frac{1}{L_{\min}}\|\nabla \phi(x^t)\|^2
        \\&\leq \frac{2L_{\max}}{L_{\min}}(\phi(x^t) - \phi(x^*))
        \\&= \frac{2L_{\max}}{L_{\min}}(h(x^t) - h(x^*)).
    \end{align*}
\end{proof}

\begin{lemma}[Three-point property with constraint \citep{tse08}] \label{lemma:threePoints}
    Let $u$ be a convex function, and let $D_\Phi(\cdot, \cdot)$ be the Bregman divergence induced by the mirror map $\Phi$. Given a convex constraint set $C \in \Real^d$. For a given vector $z$, let 
    \begin{align} \label{eq:defzplus}
        z^+ \coloneqq \underset{x \in C}{\arg\min} \{ u(x) + D_\Phi(x, z) \}. 
    \end{align}
    Then 
    \begin{align}
        u(x) + D_\Phi(x, z) \geq u(z^+) + D_\Phi(z^+, z) + D_\Phi(x, z^+) \quad \forall x \in C. \label{eq:theePointIneq}
    \end{align}
\end{lemma}

\begin{lemma}[Descent lemma] \label{lemma:descentLemma}
     Under \Cref{ass:hStruncture}. Given $x \in \Real^n$, let $y \coloneqq x - G_I \nabla h(x) $, then
    \begin{align*}
       \mE[ h(y) ] ~\leq~ h(x) - \frac{1}{2} \frac{\tau-1}{n-1} \| G_{[n]} \nabla h(x) \|_{\bL}^2,
    \end{align*}
    and
    \begin{align*}
        \mE[ h( y ) ] - h(x^*) \leq \left(1 - \frac{\mu_{\min}}{L_{\max}} \right) \mE[ h( x ) ] - h(x^*).
    \end{align*}
\end{lemma}

\begin{proof}
    % By \citet[Equation~7]{necoara2017random}.
    For the first inequality,
    \begin{align*}
       \mE[ h(y) ] &~\leq~ \mE \left[ h(x) + \langle \nabla h(x), y - x \rangle + \frac{1}{2} \| y - x \|_{\bL}^2 \right] \\
       &~=~ h(x) - \mE \left[ \langle \nabla h(x), G_I \nabla h(x) \rangle + \frac{1}{2}  \| G_I \nabla h(x) \|_{\bL}^2 \right] \\
       &~=~ h(x) - \frac{1}{2} \frac{\tau-1}{n-1} \| G_{[n]} \nabla h(x) \|_{\bL}^2 \quad \text{(By \cref{eq:GnormExp} and \cref{eq:GeqGTLG})}.
    \end{align*}
    For the second inequality, apply \cref{lemma:GnormLowerAndUpperBound} to the above equation, we get
    \begin{align*}
        \mE[ h(y) ] &~\leq~ h(x) - \frac{\tau-1}{n-1} \frac{\mu_{\min}}{L_{\max}} ( h( x ) - h(x^*) ) \\
        \Rightarrow~~ \mE[ h(y) ] - h(x^*) &~\leq \left(1 -  \frac{\tau-1}{n-1} \frac{\mu_{\min}}{L_{\max}} \right) ( h(x) - h(x^*) ). 
    \end{align*}
\end{proof}

\section{Inexact randomized block coordinate descent} \label{appendix:inexact_rbcd}
% There is a line of research on the convergence of randomized CD with inexact gradient information~\cite{tbd}. 

\begin{algorithm}[t]
 \DontPrintSemicolon
 \SetKwComment{tcp}{\tiny [}{]}
 \caption{Inexact Random Block Coordinate Descent Method with Linear Constraint\label{alg:inexactrbcd}}
 \smallskip
 \KwIn{number of participating coordinates in each round $1<\tau\leq n$, $\delta \in (0,1)$}
 $x_i^{(0)} \gets 0$ for all $i \in [n]$, let $r =\frac{\tau-1}{n-1}$. \tcp*{\tiny initialization}
 \For{\nllabel{alg-level-set-loop}$t\gets0,1,\ldots,T$}{
    $I \leftarrow \mbox{ random set of } \tau \mbox{ coordinates }$\tcp*{\tiny randomly select a subset of coordinates}
    $\wt{\nabla} h(x^{(t)}) \gets 0$ \tcp*{\tiny }
    \For{\nllabel{alg-level-set-loop}$i\in I$}{
        $\wt{\nabla}_i h(x^{(t)}) \gets \texttt{oracle}(\wt{\nabla} h_i(x^{(t-1)}), \delta/2 )$ 
    }
    $x^{(t+1)} \coloneqq x^{(t)} - \eta_t G_{I_t} \wt{\nabla} h(x^{(t)})$ \tcp*{\tiny update iterate}
 }
 \Return{$x^{(T)}$}
\end{algorithm}

In this section, we analyze the convergence behaviour of RBCD with inexact gradient oracle for solving \cref{eq:CDObj}. The detailed algorithm is shown in \Cref{alg:inexactrbcd} and the convergence rate is shown in the following theorem. 
\begin{theorem} \label{thm:inexact_rbcd}
    Assume \Cref{ass:hStruncture} is satisfied. Let
    \begin{align}
        \kappa &~=~ \frac{ r \mu }{ 8 L_{\max}} \min \left\{ \frac{1}{4}, \frac{L_{\min}^{\frac{3}{2}}}{4 r^{\frac{1}{2}} L_{\max}^{\frac{3}{2}} } \right\} ~\in (0,1) \nonumber \\
        \delta &~=~ (1-\kappa) / 2, \nonumber \\
        \eta &~=~ \min \left\{ \frac{1}{4}, \frac{L_{\min}^{\frac{3}{2}}}{4 r^{\frac{1}{2}} L_{\max}^{\frac{3}{2}} } \right\}
    \end{align} 
    Let $x^{(T)}$ be the iterate generated from \Cref{alg:inexactrbcd}, then we have 
    \[\mE\left[h(x^{(T)})\right] - h(x^*) \leq (1-\kappa)^T\left[h(x^{(0)}) - h(x^*) \right]. \]
\end{theorem}

Our analysis is enlightened by a recent work from \citet{LiuSY21}, who extended the distributed dual accelerated gradient algorithm~\citep{Scaman2017OptimalAF} with \emph{lazy dual gradient}. Before presenting the proof of \cref{thm:inexact_rbcd}, we first cite an important lemma from \citep{LiuSY21} which gives a bound on the difference between inexact and exact gradients. 
\begin{lemma}[\protect{\citet[Lemma~1]{LiuSY21}}] \label{lemma:LiuLemma1}
    Given $\delta > 0$. Let $\{ x^{(t)} \}_{t=0}^{\infty}$ and $\{ \wt{\nabla} h(x^{(t)}) \}_{t=0}^{\infty}$ be generated from \Cref{alg:inexactrbcd} with $\delta/2$-inexact gradient oracle. Then $\forall t \in \mathbb{N}$,
    \begin{align}
        \mE\left[\| \wt{\nabla} h(x^{(t)}) - {\nabla} h(x^{(t)}) \|^2 \right] \leq \sum_{j=0}^{t-1} \delta^{t-j} \mE \left[ \| \nabla h(x^{(j)}) - {\nabla} h(x^{(j+1)}) \|^2 \right]. \label{eq:recursiveGradDiff}
    \end{align}
\end{lemma}
\begin{proof}
An alternative proof can be found in~\citet[Appendix A]{LiuSY21}. Here we reproduce it for completeness. By the definition of $\delta$-inexact gradient oracle and the warm start point, we have
\begin{align*}
    \mE\left[\| \wt{\nabla} h(x^{(t)}) - {\nabla} h(x^{(t)}) \|^2 ~\Big|~ x^{(t)} \right] &\leq \frac{\delta}{2} \| \wt{\nabla} h(x^{(t-1)}) - {\nabla} h(x^{(t)}) \|^2 \\
    &\leq \delta \| \wt{\nabla} h(x^{(t-1)}) - {\nabla} h(x^{(t-1)}) \|^2 + \delta \| {\nabla} h(x^{(t-1)}) - {\nabla} h(x^{(t)}) \|^2 .
\end{align*}
Take expectation on both sides of the above inequality and apply it recursively, we obtain the desired result.
\end{proof}

Next, we show the proof for \Cref{thm:inexact_rbcd}. 
\begin{proof}
    % The proof exhibits the following sketch:
    % \begin{itemize}
    %     \item First, we construct a Lyapunov function $L_t = \mE[ h(x^{(t)}) - h(x^*) + A_t]$ for some $A_t \geq 0 ~\forall t \in \mathbb{N}$. 
    %     \item Next, we prove that a careful choice of $\eta$ and $\delta$ can lead to a linear convergence rate for the Lyapunov function, i.e., there exist $\kappa \in (0,1)$ such that $L_{t+1} \leq (1 - \kappa) L_t~\forall t \in \mathbb{N}$.
    %     \item Finally, we show that the linear convergence of the Lyapunov function implies that $\mE[ h(x^{(t)}) ]$ also converges linearly. 
    % \end{itemize}
    First, we construct a Lyapunov function 
    \[L_t \coloneqq \mE[ h(x^{(t)}) - h(x^*) + A_t],\]
    where 
    \[A_t \coloneqq  M \sum_{j=0}^{t-1} \left(\frac{1-\kappa}{2}\right)^{t-1-j} \| \nabla h(x^{(j)}) - {\nabla} h(x^{(j+1)}) \|^2 \quad \forall t \geq 1, \quad A_0 = 0\]
    for some $M > 0$ and $\kappa \in (0,1)$ which we will define later. 
    
    Next, we show that a careful choice of hyperparameters can lead to a linear convergence rate of the Lyapunov function. We begin with the smoothness property of $h$,
    \begin{align}
        &\mE \left[ h( x^{(t+1)} ) ~\Big|~ x^{(t)} \right] \nonumber \\
        \leq~& h( x^{(t)} ) - \langle \nabla h(x^{(t)}), r \eta G_{[n]} \wt{\nabla} h(x^{(t)}) \rangle + \frac{r \eta^2}{2} \|  G_{[n]} \wt{\nabla} h(x^{(t)}) \|_{\bL}^2 \nonumber \\
        \leq~& h(x^{(t)}) - r \eta \| G_{[n]} \nabla h(x^{(t)}) \|_{\bL}^2  -  \vphantom{\frac{r \eta^2 L_{\max}}{2}} r \eta \langle \nabla h(x^{(t)}), G_{[n]} ( \wt{\nabla} h(x^{(t)}) - \nabla h(x^{(t)}) ) \rangle + \frac{r \eta^2 }{2} \| G_{[n]} \wt{\nabla} h(x^{(t)})\|_{\bL}^2 \nonumber \\
        \leq~& h(x^{(t)}) - r \eta \| G_{[n]} \nabla h(x^{(t)}) \|_{\bL}^2 + r \eta \left( \frac{1}{2} \| \bL^{\frac{1}{2}} G_{[n]} \nabla h(x^{(t)}) \|_{2}^2 + \frac{1}{2} \| \bL^{\frac{1}{2}} G_{[n]} ( \wt{\nabla} h(x^{(t)}) - \nabla h(x^{(t)}) )  \|_2^2 \right) \nonumber \\
        &+ \frac{r \eta^2 }{2} \| G_{[n]} ( \wt{\nabla} h(x^{(t)}) - \nabla h(x^{(t)}) + \nabla h(x^{(t)}) ) \|_{\bL}^2 \nonumber \\
        \leq~& h(x^{(t)}) - r \eta \| G_{[n]} \nabla h(x^{(t)}) \|_{\bL}^2 + r \eta \left( \frac{1}{2} \| G_{[n]} \nabla h(x^{(t)}) \|_{\bL}^2 + \frac{1}{2} \| G_{[n]} ( \wt{\nabla} h(x^{(t)}) - \nabla h(x^{(t)}) )  \|_{\bL}^2 \right) \nonumber \\
        &+ r \eta^2 \| G_{[n]} ( \wt{\nabla} h(x^{(t)}) - \nabla h(x^{(t)} ) ) \|_{\bL}^2 + r \eta^2 \| G_{[n]} \nabla h(x^{(t)}) \|_{\bL}^2.
        \label{eq:decompose}
    \end{align}
    By taking the expectation of both sides with respect to $x^{(t)}$, we get
    \begin{align}
        \mE\left[ h(x^{(t+1)}) \right]
        ~\leq~&\mE\left[ h(x^{(t)}) - \left( \frac{r \eta }{2} - r \eta^2 \right) \| G_{[n]} \nabla h(x^{(t)}) \|_{\bL}^2 + \left( \frac{r \eta}{2} + r \eta^2 \right) \| G_{[n]} ( \wt{\nabla} h(x^{(t)}) - \nabla h(x^{(t)} ) \|_{\bL}^2 \right] \nonumber \\
        ~\leq~&\mE\bigg[ h(x^{(t)}) - \left( \frac{r \eta }{2} - r \eta^2 \right) \| G_{[n]} \nabla h(x^{(t)}) \|_{\bL}^2 \nonumber
        \\&+ \left( \frac{r \eta}{2} + r \eta^2 \right) \frac{L_{\max}}{L_{\min}^2} \sum_{j=0}^{t-1} \delta^{t-j} \| \nabla h(x^{(j)}) - {\nabla} h(x^{(j+1)}) \|^2 \bigg], \label{eq:decompose2}
    \end{align}
    where the last inequality follows from \Cref{lemma:eigenG} and \Cref{lemma:LiuLemma1}. In order to establish the linear convergence of the Lyapunov function, we add and substract the term $\alpha \mE[ \| \nabla h(x^{(t)}) - {\nabla} h(x^{(t+1)}) \|^2 ]$ to the RHS of \cref{eq:decompose2}, where $\alpha > 0$ is some positive scalar. Then \cref{eq:decompose2} becomes
    \begin{align}
        \mE\left[ h(x^{(t+1)}) \right] \leq~ &\mE\bigg[ h(x^{(t)}) - \left( \frac{r \eta }{2} - r \eta^2 \right) \| G_{[n]} \nabla h(x^{(t)}) \|_{\bL}^2 + \alpha \| \nabla h(x^{(t)}) - {\nabla} h(x^{(t+1)}) \|^2 \nonumber \\
        &+ \left( \frac{r \eta}{2} + r \eta^2 \right) \frac{L_{\max}}{L_{\min}^2} \sum_{j=0}^{t-1} \delta^{t-j} \| \nabla h(x^{(j)}) - {\nabla} h(x^{(j+1)}) \|^2 - \alpha \| \nabla h(x^{(t)}) - {\nabla} h(x^{(t+1)}) \|^2 \bigg]. \label{eq:decompose3}
    \end{align}
    We also know that
    \begin{align*}
        &\mE \left[ \alpha \| \nabla h(x^{(t)}) - {\nabla} h(x^{(t+1)}) \|^2 \right] \\ 
        ~\leq~&  \mE \left[ \frac{\alpha}{L_{\min}} \| \nabla h(x^{(t)}) - {\nabla} h(x^{(t+1)}) \|_{\bL}^2 \right]  \\
        ~\leq~&  \mE \left[ \frac{\alpha L_{\max}^2}{L_{\min}} \| \eta G_{I_t} \wt{\nabla} h( x^{(t)} )  \|_{\bL}^2 \right]  \qquad \text{(By $L_{\max}$-smoothness of $h$)}  \\
        ~\leq~& \frac{\alpha r L_{\max}^2 \eta^2}{L_{\min}} \mE \left[  \| G_{[n]} \wt{\nabla} h( x^{(t)} )  \|_{\bL}^2 \right] \qquad \text{(By \cref{eq:GnormExp})} \\
        ~\leq~&  \frac{2 \alpha r L_{\max}^2 \eta^2}{L_{\min}} \mE \left[ \| G_{[n]} \nabla h( x^{(t)} ) \|^2_{\bL} \right] + \frac{2 \alpha r L_{\max}^2 \eta^2}{L_{\min}} \mE \left[  \| G_{[n]} ( \wt{\nabla} h( x^{(t)} ) - \nabla h( x^{(t)} ) ) \|^2_{\bL}  \right] \\
        ~\leq~&  \frac{2 \alpha r L_{\max}^2 \eta^2}{L_{\min}} \mE \left[ \| G_{[n]} \nabla h( x^{(t)} ) \|^2_{\bL} \right] + \frac{2 \alpha r \eta^2 L_{\max}^3}{ L_{\min}^3 } \mE \left[  \| \wt{\nabla} h( x^{(t)} ) - \nabla h( x^{(t)} ) \|^2  \right] \\
        ~\leq~& \frac{2 \alpha r L_{\max}^2 \eta^2}{L_{\min}} \mE \left[ \| G_{[n]} \nabla h( x^{(t)} ) \|^2_{\bL} \right] + \frac{2 \alpha r \eta^2 L_{\max}^3 }{ L_{\min}^3 }  \mE \left[ \sum_{j=0}^{t-1} \delta^{t-j} \| \nabla h(x^{(j)}) - {\nabla} h(x^{(j+1)}) \|^2 \right],
    \end{align*}
    the last inequality is from \Cref{lemma:LiuLemma1}.
    Now we can determine the hyperparameters in the Lyapunov function by plugging the above inequality into~\cref{eq:decompose3} 
    \begin{align}
        &\mE\left[ h(x^{(t+1)}) \right] \nonumber \\
        ~\leq~& \mE\Bigg[ h(x^{(t)}) - {\left( \frac{r \eta }{2} - r \eta^2 - \frac{2 \alpha r L_{\max}^2 \eta^2}{L_{\min}} \right)} \| G_{[n]} \nabla h(x^{(t)}) \|_{\bL}^2  \nonumber \\
        &+ {\left( \frac{r \eta}{2} + r \eta^2 + \frac{2 \alpha r \eta^2 L_{\max}^2}{ L_{\min} } \right) \frac{L_{\max}}{L_{\min}^2} \sum_{j=0}^{t-1} \delta^{t-j} \| \nabla h(x^{(j)}) - {\nabla} h(x^{(j+1)}) \|^2 - \alpha \| \nabla h(x^{(t)}) - {\nabla} h(x^{(t+1)}) \|^2} \Bigg]  \nonumber \\
        ~\leq~& \mE\Bigg[ h(x^{(t)}) - \underbrace{\left( \frac{r \eta }{2} - r \eta^2 - \frac{2 \alpha r L_{\max}^2 \eta^2}{L_{\min}} \right) \frac{2\mu}{ L_{\max}} }_{\kappa} ( h(x^{(t)}) - h(x^*) ) \qquad \text{(By~\Cref{lemma:GnormLowerAndUpperBound})} \nonumber \\
        &+ \underbrace{\left( \frac{r \eta}{2} + r \eta^2 + \frac{2 \alpha r \eta^2 L_{\max}^2}{ L_{\min} } \right) \frac{L_{\max}}{L_{\min}^2} \sum_{j=0}^{t-1} \delta^{t-j} \| \nabla h(x^{(j)}) - {\nabla} h(x^{(j+1)}) \|^2 - \alpha \| \nabla h(x^{(t)}) - {\nabla} h(x^{(t+1)}) \|^2}_{ \overset{\rm{(i)}}{\leq} (1 - \kappa) A_t - A_{t+1}} \Bigg],  \label{eq:recursion}
    \end{align}
    where we let $M \coloneqq \left( \frac{r \eta}{2} + r \eta^2 + \frac{2 \alpha r \eta^2 L_{\max}^2}{ L_{\min} } \right) \frac{L_{\max}}{L_{\min}^2}$ and $\alpha \geq M$, (i) comes from the following derivation
    \begin{align*}
        &\left( \frac{r \eta}{2} + r \eta^2 + \frac{2 \alpha r \eta^2 L_{\max}^2}{ L_{\min} } \right) \frac{L_{\max}}{L_{\min}^2} \sum_{j=0}^{t-1} \delta^{t-j} \| \nabla h(x^{(j)}) - {\nabla} h(x^{(j+1)}) \|^2 - \alpha \| \nabla h(x^{(t)}) - {\nabla} h(x^{(t+1)}) \|^2 \\
        \leq~& M \sum_{j=0}^{t-1} \delta^{t-j} \| \nabla h(x^{(j)}) - {\nabla} h(x^{(j+1)}) \|^2 - M \| \nabla h(x^{(t)}) - {\nabla} h(x^{(t+1)}) \|^2 \\
        =~& M \sum_{j=0}^{t-1} \left( \frac{1-\kappa}{2} \right)^{t-j} \| \nabla h(x^{(j)}) - {\nabla} h(x^{(j+1)}) \|^2 - M \| \nabla h(x^{(t)}) - {\nabla} h(x^{(t+1)}) \|^2 \quad \text{(By the definition of $\delta$)} \\
        =~& (1 - \kappa) A_t - A_{t+1} \qquad \text{(By definition of $A_t$)}.
    \end{align*}
    % Next, we construct the Lyapunov function, we want to find $\eta, \alpha, A_t > 0$ (in order the guarantee sufficient decrease of the Lyapunov function) such that $\mathrm{Term 2} \in (0,1)$ and 
    % \[
    %   \mathrm{Term~2} \leq (1 - \kappa) A_t - A_{t+1}, \qquad \kappa \coloneqq \mathrm{Term~1} = \left( \frac{r \eta }{2} - r \eta^2 - \frac{2 \alpha r L_{\max}^2 \eta^2}{L_{\min}} \right) \frac{2\mu}{ L_{\max}}.
    % \]
    
    % Let  We define $A_t$ as
    % \begin{align*}
    %     A_t &~\coloneqq~   M \sum_{j=0}^{t-1} \left(\frac{1-\kappa}{2}\right)^{t-1-j} \| \nabla h(x^{(j)}) - {\nabla} h(x^{(j+1)}) \|^2 \quad \forall t \geq 1, \quad A_0 = 0.
    % \end{align*}
    % By this construction of $A_t$, we have
    % \begin{align*}
    %     &(1-\kappa)A_t - A_{t+1} \\
    %     ~=~&   M \sum_{j=0}^{t-1} 2 \left(\frac{1-\kappa}{2}\right)^{t-j} \| \nabla h(x^{(j)}) - {\nabla} h(x^{(j+1)}) \|^2 - M \sum_{j=0}^{t} \left(\frac{1-\kappa}{2}\right)^{t-j} \| \nabla h(x^{(j)}) - {\nabla} h(x^{(j+1)}) \|^2 \\
    %     ~=~& M \sum_{j=0}^{t-1} \left(\frac{1-\kappa}{2}\right)^{t-j} \| \nabla h(x^{(j)}) - {\nabla} h(x^{(j+1)}) \|^2 - M \| \nabla h(x^{(t)}) - {\nabla} h(x^{(t+1)}) \|^2 \\
    %     ~\geq~& M \sum_{j=0}^{t-1} \delta^{t-j} \| \nabla h(x^{(j)}) - {\nabla} h(x^{(j+1)}) \|^2 - M \| \nabla h(x^{(t)}) - {\nabla} h(x^{(t+1)}) \|^2 \quad \mathrm{( By~} \delta \leq{(1-\kappa)}/{2} ).
    % \end{align*}
    % Therefore we can obtain $(1-\kappa) A_t - A_{t+1} \geq \mathrm{Term 2}$ if $\alpha \geq M$.
    Then we need to find $\eta, \alpha$ that satisfy the following conditions:
    \begin{align}
        \alpha &\geq M = \left( \frac{r \eta}{2} + r \eta^2 + \frac{2 \alpha r \eta^2 L_{\max}^2}{ L_{\min} } \right) \frac{L_{\max}}{L_{\min}^2}, \label{eq:alpha_condition} \\
        \kappa & = \left( \frac{r \eta }{2} - r \eta^2 - \frac{2 \alpha r L_{\max}^2 \eta^2}{L_{\min}} \right) \frac{2\mu}{ L_{\max}} \in (0,1). \label{eq:eta_condition}
    \end{align}
    Indeed, we let 
    \begin{align}
        \eta = \min \left\{ \frac{1}{4}, \frac{L_{\min}^{\frac{3}{2}}}{4 r^{\frac{1}{2}} L_{\max}^{\frac{3}{2}} } \right\} \quad \text{and} \quad \alpha = 2\left( \frac{r \eta}{2} + r \eta^2 \right) \frac{L_{\max}}{L_{\min}^2}. \label{eq:eta_alpha_def}
    \end{align}
    We show that the above construction of $\eta$ and $\alpha$ satisfies \cref{eq:alpha_condition} and \cref{eq:eta_condition}. For \cref{eq:alpha_condition},
    \begin{align*}
        \alpha &~=~ 2\left( \frac{r \eta}{2} + r \eta^2 \right) \frac{L_{\max}}{L_{\min}^2}  \\
        & ~=~ \left( \frac{r \eta}{2} + r \eta^2 \right) \frac{L_{\max}}{L_{\min}^2} + \frac{1}{2} \alpha \\
        &~\geq~ \left( \frac{r \eta}{2} + r \eta^2 \right) \frac{L_{\max}}{L_{\min}^2} + \left( \frac{2 r \eta^2 L_{\max}^2}{ L_{\min} } \right) \frac{L_{\max}}{L_{\min}^2} \alpha \qquad \text{(By definition of $\eta$)} \\
        &~=~ M.
    \end{align*}
    For \cref{eq:eta_condition},
    \begin{align*}
        \kappa &~=~ \left( \frac{r \eta }{2} - r \eta^2 - \frac{2 \alpha r L_{\max}^2 \eta^2}{L_{\min}} \right) \frac{2\mu}{ L_{\max}} \\
        &~=~ \left( \frac{r \eta }{2} - r \eta^2 - 2\left( \frac{r \eta}{2} + r \eta^2 \right) \frac{L_{\max}}{L_{\min}^2} \frac{2  r L_{\max}^2 \eta^2}{L_{\min}} \right) \frac{2\mu}{ L_{\max}} \quad \text{(By definition of $\alpha$)} \\
        &~\geq  \frac{2\mu}{ L_{\max}}~ \left( \frac{r \eta }{2} - r \eta^2 - 2\left( \frac{r \eta}{2} + r \eta^2 \right) \frac{1}{8} \right) \qquad \text{(By definition of $\eta$)} \\
        &~=~  \frac{2\mu}{ L_{\max}} \left( \frac{3r \eta}{8} - \frac{5}{4} r \eta^2 \right)  \\
        &~\geq~  \frac{2\mu}{ L_{\max}} \left( \frac{3r \eta}{8} - \frac{5}{4} \frac{r \eta}{4} \right) \qquad \text{(By the condition $\eta \leq \frac{1}{4}$)} \\
        &~\geq~  \frac{2\mu}{ L_{\max}} \frac{1}{16} r \eta \\
        &~=~  \frac{ r \mu }{ 8 L_{\max}} \min \left\{ \frac{1}{4}, \frac{L_{\min}^{\frac{3}{2}}}{4 r^{\frac{1}{2}} L_{\max}^{\frac{3}{2}} } \right\} \qquad \text{(By definition of $\eta$)}.
    \end{align*}
    On the other hand, 
    \begin{align*}
        \kappa &~=~ \left( \frac{r \eta }{2} - r \eta^2 - \frac{2 \alpha r L_{\max}^2 \eta^2}{L_{\min}} \right) \frac{2\mu}{ L_{\max}} \\
        &~\leq~ \frac{r \eta }{2} \frac{2\mu}{ L_{\max}} \\
        &~\leq~ \frac{r \mu}{4 L_{\max}}  \qquad \text{(By the condition $\eta \leq \frac{1}{4}$)} \\
        &~<~ 1.
    \end{align*}
    Therefore we proved that $\kappa$ satisfy the condition
    \begin{align}
      \frac{r \mu}{ 8 L_{\max}} \min \left\{ \frac{1}{4}, \frac{L_{\min}^{\frac{3}{2}}}{4 r^{\frac{1}{2}} L_{\max}^{\frac{3}{2}} } \right\} \leq  \kappa < 1. \label{eq:kappa_condition}
    \end{align}
    Go back to \cref{eq:recursion}, we get
    \begin{align}
        \mE \left[ h( x^{(t+1)} ) \right] ~\leq~ \mE\left[ h(x^{(t)}) - \kappa ( h(x^{(t)}) - h(x^*) ) + (1 - \kappa) A_t - A_{t+1} \right].
    \end{align}
    Rearrange the above inequality, we obtain
    \[
        \mE \left[ h( x^{(t+1)} ) - h(x^*) + A_{t+1} \right] \leq (1 - \kappa)  \mE \left[ h( x^{(t)} ) - h(x^*) + A_{t} \right].
    \]
    Finally, by plugging $t = T$ and the fact that $A_0 = 0$, we have 
    \[\mE \left[ h( x^{(T)} )\right] - h(x^*) \leq (1 - \kappa)^T[h(x^{(0)}) - h(x^*)].\]
\end{proof}

\section{Accelerated randomized block coordinate descent} \label{appendix:accelerated_rbcd}

\begin{algorithm}[tbh]
 \DontPrintSemicolon
 \SetKwComment{tcp}{\tiny [}{]}
 \caption{Accelerated Random Block Coordinate Descent Method with Linear Constraint (for strongly convex objective)\label{alg:accelerated_rbcd}}
 \smallskip
 \KwIn{number of selected coordinates in each round $1<\tau\leq n$, strong convexity parameter $\mu > 0$.}
 $x_i^{(0)} \gets 0, z_i^{(0)} \gets 0$ for all $i \in [n]$, let $r =\frac{\tau-1}{n-1}$, $a = \frac{\sqrt{\mu/L_{\max}}}{\frac{1}{r} + \sqrt{\mu/L_{\max}}}, b = \frac{\mu a r^2}{L_{\max}}$. \tcp*{\tiny initialization}
 \For{\nllabel{alg-level-set-loop}$t\gets0,1,2,\ldots$}{
    $y^{(t)} \coloneqq (1 - a) x^{(t)} + a z^{(t)}$ \tcp*{\tiny standby}
   $I_t^1 \leftarrow \mbox{ random set of } \tau \mbox{ clients }$\tcp*{\tiny randomly select a subset of clients}
   $x^{(t+1)} = y^{(t)} -  G_{I_1}\nabla h( y^{(t)} )$ \tcp*{\tiny standby}
   $u^{(t)} = \frac{a^2}{a^2 +b} z^{(t)} + \frac{b}{a^2 +b} y^{(t)}$ \tcp*{\tiny standby}
   $I_t^2 \leftarrow \mbox{ random set of } \tau \mbox{ clients }$\tcp*{\tiny randomly select a subset of clients}
   $z^{(t+1)} = u^{(t)} - \frac{a r}{(a^2 +b)} G_{I_2} \nabla h( y^{(t)} )$ \tcp*{\tiny standby}
 }
\end{algorithm}

In this section, we analyze the convergence behaviour of accelerated RBCD for solving \cref{eq:CDObj}. The detailed algorithm is shown in \Cref{alg:accelerated_rbcd} and the convergence rate is shown in the following theorem. 

\begin{theorem} \label{thm:accelerated_rbcd}
    Assume that $h$ is $\mu$-strongly convex, let $\{x^{(t)}\}_{t=0}^\infty$, $\{y^{(t)}\}_{t=0}^\infty$ and $\{z^{(t)}\}_{t=0}^\infty$ be the iterates generated from \Cref{alg:accelerated_rbcd}, then
    \begin{align*}
        \mE \left[ h( x^{(T)} ) - h(x^*) \right] \leq \left( 1 - \frac{ \sqrt{\frac{\mu}{L_{\max}}}}{ \frac{1}{r} + \sqrt{\frac{\mu}{L_{\max}}} } \right)^T \left( h(x^{(0)}) - h(x^*) \right)
    \end{align*}
    for all $T \in \mathbb{N}$.
\end{theorem}

Our analysis follows the proof template from \citet{Lu18} and we do not claim much novelty for the proof technique used here. First, we prove three lemmas that correspond to \citep[Lemma~A.1, Lemma~A.2, Lemma~A.3]{Lu18}.

\begin{lemma}[\protect{\citealp[Lemma~A.1]{Lu18}}] \label{lemma:Lu18A1}
    \begin{align*}
        a^2 \| x - z^{(t)} \|^2_{\bL} + b  \| x - y^{(t)} \|^2_{\bL} ~=~ (a^2 + b) \| x - u^{(t)} \|^2_{\bL} + \frac{a^2b}{a^2+b} \| y^{(t)} - z^{(t)} \|^2_{\bL} \quad \forall x \in \Real^n.
    \end{align*}
\end{lemma}
\begin{proof}
     The proof is exactly the same as \cite[Lemma~A.1]{Lu18} since our definition of $u^{(t)}$ is the same as the definition of $u^{(t)}$ in \cite[Algorithm~2]{Lu18}.
\end{proof}

\begin{lemma}[\protect{\citealp[Lemma~A.2]{Lu18}}] \label{lemma:Lu18A2}
    Define $v^{(t+1)} \coloneqq u^{(t)} - \frac{ar}{a^2 + b} G_{[n]} \nabla h(y^{(t)})$, then
    \begin{align*}
        \| x^* - v^{(t+1)} \|^2_{\bL} - \| x^* - u^{(t)} \|^2_{\bL} &~=~ \frac{1}{r} \mE \left[ \| x^* - z^{(t+1)} \|^2_{\bL} - \| x^* - u^{(t)} \|^2_{\bL} \right]. 
    \end{align*}
\end{lemma}

\begin{proof}
     Our proof is based on minor modification of the proof for \citealp[Lemma~A.2]{Lu18}.
     \begin{align*}
          &\| x^* - v^{(t+1)} \|^2_{\bL} - \| x^* - u^{(t)} \|^2_{\bL}  \\
          ~=~& 2 \langle x^* - u^{(t)}, u^{(t)} - v^{(t+1)} \rangle_{\bL} + \| u^{(t)} - v^{(t+1)} \|_{\bL}^2 \\
          ~=~& 2 \left\langle x^* - u^{(t)}, \frac{ar}{a^2 + b} G_{[n]} \nabla h(y^{(t)}) \right\rangle_{\bL} + \left\| \frac{ar}{a^2 + b} G_{[n]} \nabla h(y^{(t)}) \right\|_{\bL}^2 \quad \text{(By definition of $v^{(t+1)}$)} \\
        ~=~& \mE_{I_t^2}\left[ \frac{2}{r} \left\langle x^* - u^{(t)}, \frac{ar}{a^2 + b} G_{I_t^2} \nabla h(y^{(t)}) \right\rangle_{\bL } + \frac{1}{r} \left\| \frac{ar}{a^2 + b} G_{I_t^2} \nabla h(y^{(t)}) \right\|_{\bL}^2 \right] \quad \text{(By \Cref{lemma:matrix_expectation})} \\
        ~=~& \frac{1}{r} \mE_{I_t^2}\left[ 2 \left\langle x^* - u^{(t)}, u^{(t)} - z^{(t+1)} \right\rangle_{\bL } + \left\| u^{(t)} - z^{(t+1)} \right\|_{\bL}^2 \right] \qquad \text{(By definition of $z^{(t+1)}$)} \\
        ~=~& \frac{1}{r}  \mE \left[ \| x^* - z^{(t+1)} \|^2_{\bL} - \| x^* - u^{(t)} \|^2_{\bL} \right].
     \end{align*}
\end{proof}

\begin{lemma}[\protect{\citealp[Lemma~A.3]{Lu18}}] \label{lemma:Lu18A3}
    \[
        a^2 \leq (1 - a)(a^2 + b).
    \]
\end{lemma}
\begin{proof}
     See the proof of \citet[Lemma~A.3]{Lu18}, we only need to substitute $n$ to $1/r$ and $\mu$ to $\mu/L_{\max}$.
\end{proof}

Now we show the proof for \cref{thm:accelerated_rbcd}. 
\begin{proof}
We define some auxiliary variables
\begin{align*}
    v^{(t+1)} \coloneqq u^{(t)} - \frac{ar}{a^2 + b} G_{[n]} \nabla h(y^{(t)}).
\end{align*}

We start with expected decrease from $y^{(t)}$ to $x^{(t+1)}$:
     \begin{align*}
         & \mE \left[ h( x^{(t+1)} ) - h(y^{(t)}) ~\Big|~ y^{(t)}  \right] \\
         \leq~& -\frac{r}{2} \| G_{[n]} \nabla h( y^{(t)} ) \|_{\bL}^2  \qquad \text{(By \Cref{lemma:descentLemma})} \\
         \leq~& a \langle G_{[n]} \nabla h( y^{(t)} ), v^{(t+1)} - z^{(t)} \rangle_{\bL} + \frac{a^2}{2r} \| v^{(t+1)} - z^{(t)} \|_{\bL}^2 \quad \text{(By $\frac{1}{2}\|x\|^2_{\bL} + \frac{1}{2} \|y\|^2_{\bL} \geq \ip{x}{y}_{\bL}$ )}  \\
        =~& a \langle G_{[n]} \nabla h( y^{(t)} ), v^{(t+1)} - z^{(t)} \rangle_{\bL} + \frac{a^2+b}{2r} \| v^{(t+1)} - u^{(t)} \|_{\bL}^2 \\&+ \frac{a^2b}{2r(a^2+b)} \| y^{(t)} - z^{(t)} \|_{\bL}^2 -  \frac{b}{2r} \| v^{(t+1)} - y^{(t)} \|_{\bL}^2 \qquad \text{(By \Cref{lemma:Lu18A1})} \\
        \leq~& \underbrace{a \langle G_{[n]} \nabla h( y^{(t)} ), v^{(t+1)} - z^{(t)} \rangle_{\bL} + \frac{a^2+b}{2r} \| v^{(t+1)} - u^{(t)} \|_{\bL}^2}_{\mbox{\small term 1}} + \frac{a^2b}{2r(a^2+b)} \| y^{(t)} - z^{(t)} \|_{\bL}^2.
     \end{align*}
     By the definition of $v^{(t+1)}$, we have
     \begin{align*}
         v^{(t+1)} ~=~ & \underset{z \in \Real^n}{\arg\min} ~a \langle G_{[n]} \nabla h(y^{(t)}), z - z^{(t)} \rangle_{\bL} + \frac{a^2+b}{2r} \| z - u^{(t)} \|_{\bL}^2 
     \end{align*}
     Then by applying \Cref{lemma:threePoints} to term 1, we obtain
     \begin{align}
         & \mE \left[ h( x^{(t+1)} ) - h(y^{(t)}) ~\Big|~ y^{(t)}  \right] \nonumber \\
         \leq~& a \langle G_{[n]} \nabla h( y^{(t)} ), x^* - z^{(t)} \rangle_{\bL} + \frac{a^2+b}{2r} \| x^* - u^{(t)} \|_{\bL}^2 - \frac{a^2+b}{2r} \| x^* - v^{(t+1)} \|_{\bL}^2 + \frac{a^2b}{2r(a^2+b)} \| y^{(t)} - z^{(t)} \|_{\bL}^2 \nonumber \\
         =~& a \langle \nabla h( y^{(t)} ), x^* - z^{(t)} \rangle + \frac{a^2+b}{2r} \| x^* - u^{(t)} \|_{\bL}^2 - \frac{a^2+b}{2r} \| x^* - v^{(t+1)} \|_{\bL}^2 + \frac{a^2b}{2r(a^2+b)} \| y^{(t)} - z^{(t)} \|_{\bL}^2 \quad \text{(By \cref{eq:GIpExp})} \nonumber \\
         {=}~& a \langle \nabla h( y^{(t)} ), x^* - z^{(t)} \rangle + \frac{a^2+b}{2r^2} \mE \left[ \| x^* - u^{(t)} \|_{\bL}^2 - \| x^* - z^{(t+1)} \|_{\bL}^2 \right] + \frac{a^2b}{2r(a^2+b)} \| y^{(t)} - z^{(t)} \|_{\bL}^2 \text{(By \Cref{lemma:Lu18A2})} \nonumber \\
         \leq~& a \langle \nabla h( y^{(t)} ), x^* - z^{(t)} \rangle + \frac{a^2+b}{2r^2} \mE \left[ \| x^* - u^{(t)} \|_{\bL}^2 - \| x^* - z^{(t+1)} \|_{\bL}^2 \right] + \frac{a^2b}{2r^2(a^2+b)} \| y^{(t)} - z^{(t)} \|_{\bL}^2 \nonumber \\
         =~& a \langle \nabla h( y^{(t)} ), x^* - z^{(t)} \rangle + \frac{1}{2r^2} \left( (a^2 +b ) \| x^* - u^{(t)} \|_{\bL}^2 + \frac{a^2b}{a^2+b}  \| y^{(t)} - z^{(t)} \|_{\bL}^2 \right) - \frac{a^2 + b}{2r^2} \mE \left[  \| x^* - z^{(t+1)} \|_{\bL}^2 \right] \nonumber \\
         =~& a \langle \nabla h( y^{(t)} ), x^* - z^{(t)} \rangle + \frac{1}{2r^2} \left( a^2 \| x^* - z^{(t)} \|_{\bL}^2 + b \| x^* - y^{(t)} \|_{\bL}^2 \right) - \frac{a^2 + b}{2r^2} \mE \left[  \| x^* - z^{(t+1)} \|_{\bL}^2 \right] \text{(By \Cref{lemma:Lu18A1})}. \label{eq:accRBCDStronglyCVX1}
     \end{align}
     Similar to the proof of \citet[Theorem~3.1]{Lu18}, by using the strong convexity of $h$, we have
     \begin{align*}
         h( y^{(t)} ) - h(x^*) ~\leq~& \langle \nabla h( y^{(t)} ), y^{(t)} - x^*  \rangle - \frac{\mu}{2} \| y^{(t)} - x^* \|^2 \\
         ~=~& \langle \nabla h( y^{(t)} ), y^{(t)} - z^{(t)} \rangle + \langle \nabla h( y^{(t)} ), z^{(t)} - x^* \rangle - \frac{\mu}{2} \| y^{(t)} - x^* \|^2 \\
         ~=~&  \frac{1-a}{a}\langle \nabla h( y^{(t)} ), x^{(t)} - y^{(t)} \rangle + \langle \nabla h( y^{(t)} ), z^{(t)} - x^* \rangle - \frac{\mu}{2} \| y^{(t)} - x^* \|^2 \\
         ~\leq~ & \frac{1-a}{a}( h( x^{(t)} ) - h(y^{(t)}) ) + \langle \nabla h( y^{(t)} ), z^{(t)} - x^* \rangle - \frac{\mu}{2} \| y^{(t)} - x^* \|^2 \quad \text{(By convexity of $h$)} \\
         ~\leq~ & \frac{1-a}{a}( h( x^{(t)} ) - h(y^{(t)}) ) + \langle \nabla h( y^{(t)} ), z^{(t)} - x^* \rangle - \frac{\mu}{2L_{\max}} \| y^{(t)} - x^* \|^2_{\bL}.
     \end{align*}
     By rearranging the above inequality, we get
     \begin{align*}
         h( y^{(t)} ) - h(x^*) \leq (1-a) ( h( x^{(t)} ) - h(x^*) ) + a \langle \nabla h( y^{(t)} ), z^{(t)} - x^* \rangle - \frac{\mu a}{2L_{\max}} \| y^{(t)} - x^* \|^2_{\bL}.
     \end{align*}
     Sum the above inequality with \cref{eq:accRBCDStronglyCVX1}, we get
     \begin{align*}
         &\mE \left[ h( x^{(t+1)} ) - h(x^*) ~\Big|~ y^{(t)}  \right] \\
         \leq~& (1-a) ( h( x^{(t)} ) - h(x^{*}) ) + \frac{1}{2r^2} \left( a^2 \| x^* - z^{(t)} \|_{\bL}^2 + b \| x^* - y^{(t)} \|_{\bL}^2 \right)  \\
         &~~ - \frac{a^2 + b}{2r^2} \mE \left[  \| x^* - z^{(t+1)} \|_{\bL}^2 \right] - \frac{\mu a}{2L_{\max}} \| y^{(t)} - x^* \|^2_{\bL} \\
         \leq~& (1-a) ( h( x^{(t)} ) - h(x^{*}) ) ) + \frac{1}{2r^2}  a^2 \| x^* - z^{(t)} \|_{\bL}^2 + b - \frac{a^2 + b}{2r^2} \mE \left[  \| x^* - z^{(t+1)} \|_{\bL}^2 \right] \text{(By $b = \frac{\mu a r^2}{L_{\max}}$)} \\
         \leq~& (1-a) ( h( x^{(t)} ) - h(x^{*}) ) + \frac{(1-a)(a^2+b)}{2r^2}  \| x^* - z^{(t)} \|_{\bL}^2 - \frac{a^2 + b}{2r^2} \mE \left[  \| x^* - z^{(t+1)} \|_{\bL}^2 \right] \text{(By \Cref{lemma:Lu18A3})}.
     \end{align*}
     Rearrange the above inequality,
     \begin{align*}
         & \mE \left[ h( x^{(t+1)} ) - h(x^*) + \frac{a^2 + b}{2r^2}  \| x^* - z^{(t+1)} \|_{\bL}^2  ~\bigg|~ y^{(t)}   \right] \\
         \leq~& (1-a) \left( ( h( x^{(t)} ) - h(x^{*}) ) + \frac{a^2+b}{2r^2}  \| x^* - z^{(t)} \|_{\bL}^2 \right).
     \end{align*}
     Taking the expectation on both sides and recursively apply to $t = 0, 1, \ldots, T-1$ yields the desired result.
\end{proof}

\section{Derivation of the theoretical results in the main context} \label{appendix:derivation_main_context}
In this section, we build the bridge connecting the result proved in the appendix to the theorems in the main context. Before stating the major derivations, we first introduce some useful lemmas.
\begin{lemma}[Relationship between primal and dual variables] \label{lemma:relation_primal_dual}
    Assume that \cref{assum:stronglyCvx} and \cref{assum:smooth} hold.  
    \begin{enumerate}
        \item If $w_i^{(t)} = \nabla f_i^*(y_i^{(t)})$ for all $i$ and $t$, and $\hat w^{(t)}$ are picked uniformly randomly from $\{w_i^{(t)}\mid i\in[N]\}$, then
        \begin{align}
            \mE[\|\hat w^{(t)} - w^*\|^2] \leq \frac{2}{N \alpha}\left[G(y^{(t)}) - G(y^*) \right].
        \end{align}
        \item If $w_i^{(t)} = \texttt{oracle}_{f_i^*, \delta/2}(y^{(t)}_i, w_{i}^{(t-1)})$ for all $i$ and $t$, and $\hat w^{(t)}$ are picked uniformly randomly from $\{w_i^{(t)}\mid i\in[N]\}$, then
        \[\mE[\|\hat w^{(t)} - w^*\|^2] \leq \frac{8}{N\alpha} \left\{ \delta^t\left[G(y^{(0)}) - G(y^*)\right] + \sum_{j=1}^t (\delta^{t-j} + \delta^{t-j+1})\left[G(y^{(j)}) - G(y^*)\right]\right\}.\]
    \end{enumerate}
\end{lemma}
\begin{proof}
    First, we consider the case where $w_i^{(t)} = \nabla f_i^*(y_i^{(t)})$ for all $i$. By \cref{assum:stronglyCvx}, we know that $G$ is $\frac{1}{\alpha}$ smooth and convex. It follows from \citealp[Theorem~2.1.5]{cvxopt_lecture} that 
    \begin{align}
        \sum_{i=1}^N \|w_i^{(t)} - w^*\|^2 ~=~ \| \nabla G( y^{(t)} ) - \nabla G(y^*) \|^2
        &~\leq~ \frac{2}{\alpha} \left[G(y^{(t)}) - G(y^*) -  \ip{\nabla G(y^*)}{y^{(t)} - y^*}\right] \nonumber
        \\&~\leq~ \frac{2}{\alpha} \left[G(y^{(t)}) - G(y^*) \right],  \label{eq:primalDual1}
    \end{align}
    where the second inequality follows from the fact that $\ip{\nabla G(y^*)}{y - y^*} \geq 0~\forall y \in \Cscr $ since $y^*$ is optimal for the dual problem and $y^{(t)}$ is dual feasible. It then follows that 
    \[\mE[\|\hat w^{(t)} - w^*\|^2] = \frac{1}{N}\sum_{i=1}^N \|w_i^{(t)} - w^*\|^2 \leq \frac{2}{N\alpha} \left[G(y^{(t)}) - G(y^*) \right].\]
    Next, we consider the case where $w_i^{(t)} = \texttt{oracle}_{f_i^*, \delta/2}(y^{(t)}_i, w_{i}^{(t-1)})$ for all $i$. In this case, let $\tilde w_i^{(t)} = \nabla f_i^*(y_i^{(t)})$ for all $i$ and $t$. Then by \cref{lemma:LiuLemma1}, we have 
    \begin{align*}
        \sum_{i=1}^N \|w_i^{(t)} - w^*\|^2 &\leq 2\sum_{i=1}^N \|\tilde w_i^{(t)} - w^*\|^2 + 2\sum_{i=1}^N \|\tilde w_i^{(t)} - w_i^{(t)}\|^2
        \\&\leq \frac{4}{\alpha}\left[G(y^{(t)}) - G(y^*) \right] + 2\sum_{i=1}^N \|\tilde w_i^{(t)} - w_i^{(t)}\|^2 \qquad \text{(By \cref{eq:primalDual1})}
        \\&\leq \frac{4}{\alpha}\left[G(y^{(t)}) - G(y^*) \right] + 2 \sum_{i=1}^N \sum_{j=0}^{t-1}\delta^{t-j}\|\tilde w_i^{(j)} - \tilde w_i^{(j+1)}\|^2 \qquad \text{(By \Cref{lemma:LiuLemma1})}
        \\&\leq \frac{4}{\alpha}\left[G(y^{(t)}) - G(y^*) \right] + 2  \sum_{j=0}^{t-1}\delta^{t-j} \sum_{i=1}^N \left( 2\|\tilde w_i^{(j)} -  w_i^{*}\|^2 + 2 \|\tilde w_i^{(j+1)} -  w_i^{*}\|^2 \right)
        \\&\leq \frac{4}{\alpha}\left[G(y^{(t)}) - G(y^*) \right] + \frac{8}{\alpha} \sum_{j=0}^{t-1} \delta^{t-j}\left[G(y^{(j)}) - G(y^*) \right] + \frac{8}{\alpha} \sum_{j=0}^{t-1} \delta^{t-j}\left[G(y^{(j+1)}) - G(y^*) \right]
        \\&\leq \frac{8}{\alpha} \left\{ \delta^t\left[G(y^{(0)}) - G(y^*)\right] + \sum_{j=1}^t (\delta^{t-j} + \delta^{t-j+1})\left[G(y^{(j)}) - G(y^*)\right]\right\}.
    \end{align*}
     It then follows that 
    \[\mE[\|\hat w^{(t)} - w^*\|^2] \leq \frac{8}{N\alpha} \left\{ \delta^t\left[G(y^{(0)}) - G(y^*)\right] + \sum_{j=1}^t (\delta^{t-j} + \delta^{t-j+1})\left[G(y^{(j)}) - G(y^*)\right]\right\}.\]
\end{proof}

\begin{lemma}[Bound on dual objective] \label{lemma:bound_dual_obj}
    Assume that \cref{assum:stronglyCvx} and \cref{assum:smooth} hold. Then we have 
    \[G(y^{(t)}) - G(y^*) \leq \frac{1}{2\alpha}\|y^{(t)} - y^*\|^2.\]
\end{lemma}
\begin{proof}
    By \cref{assum:stronglyCvx}, we know that $G$ is $\frac{1}{\alpha}$ smooth. It follows that 
    \begin{align*}
        G(y^{(t)}) - G(y^*) &\leq \ip{\nabla G(y^*)}{y^{(t)} - y^*} + \frac{1}{2\alpha}\|y^{(t)} - y^*\|^2 
        \\&= \ip{\nabla G(y^*)}{P_{[N]}y^{(t)} - P_{[N]}y^*} + \frac{1}{2\alpha}\|y^{(t)} - y^*\|^2
        \\&= \ip{ P_{[N]}^T \nabla G(y^*)}{y^{(t)} - y^*} + \frac{1}{2\alpha}\|y^{(t)} - y^*\|^2
        \\&= \frac{1}{2\alpha}\|y^{(t)} - y^*\|^2,
    \end{align*}
    where the last equality follows from \cref{lemma:gradient_optimal} by letting $h \coloneqq G$. 
\end{proof}

\paragraph{Proof for \cref{thm:necoaraConvergenceRate} }
\begin{proof}
    As we mentioned in the main context, the convergence rate for $G(y^{(t)}) - G(y^*)$ follows directly from \citep{necoara2017random}. We reproduce it for completeness.
    Make the identification $h = G$ (extend $h$ from coordinate-wise to block-wise), then $\mu_{\min} = 1/\beta$ and $L_{\max} = 1/\alpha$. \Cref{lemma:descentLemma} gives
    \begin{align*}
        \mE [G( y^{(t)} )] - G(y^*) ~\leq~ \left( 1 - \frac{\tau-1}{N-1} \frac{\alpha}{\beta}  \right)^t ( G( y^{(0)} ) - G(y^*) ) \qquad \forall t \in \mathbb{N}. 
    \end{align*}
    Next we derive the bound for $\|w^{(t)} - w^*\|^2$. Indeed, we have
    \begin{align*}
        \mE[\|\hat w^{(T)} - w^*\|^2] &\leq \frac{2}{N \alpha}\left[G(y^{(T)}) - G(y^*) \right]
        \\&\leq \frac{2}{N \alpha} \left( 1 - \frac{\tau - 1}{N-1} \frac{\alpha}{\beta} \right)^{T} ( G(y^{(0)}) - G(y^*) )
        \\&\leq \frac{1}{N \alpha^2} \left( 1 - \frac{\tau - 1}{N-1} \frac{\alpha}{\beta} \right)^{T} \|y^*\|^2,
    \end{align*}
    where the first and third inequalities respectively follow from \cref{lemma:relation_primal_dual} and \cref{lemma:bound_dual_obj}. 
    Furthermore, if we assume \cref{assum:dissimilar}, then we have 
    \[\mE[\|\hat w^{(T)} - w^*\|^2] \leq \frac{1}{\alpha^2} \left( 1 - \frac{\tau - 1}{N-1} \frac{\alpha}{\beta} \right)^{T} \zeta^2.\]
\end{proof}

\paragraph{Proof for \cref{thm:inexactConvergence}}
\begin{proof}
    Make the identification $h = G$ (extend $h$ from coordinate-wise to block-wise), then $\mu_{\min} = 1/\beta$ and $L_{\max} = 1/\alpha$. \Cref{thm:inexact_rbcd} gives the following convergence rate
    \begin{align}
        \mE[ G( y^{(t)} ) ] - G( y^* ) ~\leq~ \left( 1 - \kappa \right)^t [ G( y^{(0)} ) - G( y^* ) ] \qquad \forall t \in \mathbb{N}. \label{eq:inexactConverge}
    \end{align}
    Note that one minor difference is that we set $\delta = (1-\kappa)/2$ and use $\delta/2$-inexact gradient oracle in the proof of \Cref{thm:inexact_rbcd}, which is equivalent as setting $\delta = (1-\kappa)/4$ with $\delta$-inexact gradient oracle. 
    
    Next we derive the bound for $\|w^{(t)} - w^*\|^2$. Indeed, we have
    \begin{align*}
        \mE[\|\hat w^{(T)} - w^*\|^2] &\leq \frac{8}{N\alpha} \left\{ \delta^T\left[G(y^{(0)}) - G(y^*)\right] + \delta^T(1+\delta)\sum_{t=1}^T \delta^{-t}\left[G(y^{(t)}) - G(y^*)\right]\right\} \quad \text{(By \Cref{lemma:relation_primal_dual})}
        \\&\leq \frac{8\delta^T}{N\alpha} \left\{ \left[G(y^{(0)}) - G(y^*)\right] + (1+\delta)\sum_{t=1}^T \left(\frac{1 - \kappa}{\delta}\right)^t\left[G(y^{(0)}) - G(y^*)\right]\right\} \quad \text{(By \cref{eq:inexactConverge})}
        \\&= \frac{8(1+\delta)\delta^T}{N\alpha} \frac{ \left(\frac{1 - \kappa}{\delta}\right)^{T+1} - 1}{\left(\frac{1 - \kappa}{\delta}\right) - 1} \left[G(y^{(0)}) - G(y^*)\right]
        \\ &\leq \frac{40}{3N\alpha}(1-\kappa)^T  \left[G(y^{(0)}) - G(y^*)\right] \qquad \left(\mbox{by}~ \delta = \frac{1-\kappa}{4}\right) 
        \\&\leq \frac{20}{3N\alpha^2}(1-\kappa)^T  \|y^*\|^2 \qquad \text{(By \Cref{lemma:bound_dual_obj})}.
    \end{align*}
    Furthermore, if we assume \cref{assum:dissimilar}, then we have 
    \[
        \mE[\|\hat w^{(T)} - w^*\|^2] \leq \frac{20}{3\alpha^2}(1-\kappa)^T  \zeta^2.
    \]
\end{proof}

\paragraph{Proof for \cref{thm:accFedDCDStronglyCVX}}
\begin{proof}
    Make the identification $h = G$ (extend $h$ from coordinate-wise to block-wise), then $\mu_{\min} = 1/\beta$ and $L_{\max} = 1/\alpha$. \Cref{thm:inexact_rbcd}, and \cref{thm:accelerated_rbcd} gives us the convergence rate for $G(y^{(t)}) - G(y^*)$.
    We only need to derive the bound for $\|w^{(t)} - w^*\|^2$. Indeed, we have
    \begin{align*}
        \mE[\|\hat w^{(T)} - w^*\|^2] &\leq \frac{2}{N \alpha} \mE \left[G(y^{(T)}) - G(y^*) \right] \qquad \text{(By \cref{lemma:relation_primal_dual})}
        \\&\leq \frac{2}{N \alpha}\left( 1 - \frac{ \sqrt{\frac{\alpha}{\beta}} }{ \frac{N-1}{\tau -1}+ \sqrt{\frac{\alpha}{\beta}} } \right)^{T} \left[ G(y^{(0)}) - G(y^*) \right]
        \\&\leq \frac{1}{N \alpha^2}\left( 1 - \frac{ \sqrt{\frac{\alpha}{\beta}} }{ \frac{N-1}{\tau -1}+ \sqrt{\frac{\alpha}{\beta}} } \right)^{T} \|y^*\|^2 \quad \text{(By \Cref{lemma:bound_dual_obj})} .
    \end{align*}
    Furthermore, if we assume \cref{assum:dissimilar}, then we have 
    \[\mE[\|\hat w^{(T)} - w^*\|^2] \leq \frac{1}{\alpha^2}\left( 1 - \frac{ \sqrt{\frac{\alpha}{\beta}} }{ \frac{N-1}{\tau -1}+ \sqrt{\frac{\alpha}{\beta}} } \right)^{T} \zeta^2.\]
\end{proof}

\section{Complexity lower bound} \label{appendix:lower_bound}
\begin{proof}[Proof of \Cref{thm:lowerBound}]
    We follow the function used by Nesterov to prove complexity lower bound for smooth and strongly convex objectives \citep{nemirovsky1983problem,cvxopt_lecture,Bubeck15}. We divide our analysis into two cases: $\alpha < \beta/N$ and $\alpha \geq \beta/N$. 
    
    First we discuss the case when $\alpha < \beta/N$. We construct $N$ functions as follows:
    \begin{align}
        f_i( w ) = \frac{\beta - N \alpha}{8} \left( w^T M^{(i)} w - \mathbf{1}_{i=1} \cdot 2 \langle e_1, w \rangle \right) + \frac{\alpha}{2} \| w \|^2 \qquad \forall i \in [N],
    \end{align}
    where $M^{(i)} : \ell_2 \to \ell_2 $ are infinite dimensional block diagonal matrix. For $k \geq 2$, we let
    \begin{align}
        \begin{pmatrix}
            M^{(k)}_{i,j} & M^{(k)}_{i,j+1} \\
            M^{(k)}_{i+1,j} & M^{(k)}_{i+1,j+1}
        \end{pmatrix}
        = 
        \begin{pmatrix}
            1 & -1 \\
            -1 & 1
        \end{pmatrix}
        \quad \text{when $i =j = pn+k$ for some $p \in \mathbb{N}$},
    \end{align}
    and $M^{(k)}_{i,j} = 0$ otherwise. For $k=1$, we follow the same construction expect that we modify its first block as
    \[
        \begin{pmatrix}
            M^{(1)}_{1,1} & M^{(1)}_{1,2} \\
            M^{(1)}_{2,1} & M^{(1)}_{2,2}
        \end{pmatrix}
        = 
        \begin{pmatrix}
            2 & -1 \\
            -1 & 1
        \end{pmatrix}.
    \]
    By this construction, it is easy to verify that $ 0 \preceq M^{(k)} \preceq 4 \mathbb{I}$, therefore $f_i$'s are $\alpha$-strongly convex and $\beta$-smooth.
    
    Further, we know that
    \[
        \sum_{k=1}^{\infty} M^{(k)} = \begin{pmatrix}
            2 & -1 & 0 & \dots \\
            -1 & 2 & -1 & \dots \\
            0 & -1 & 2 & \dots \\
            \vdots & \vdots & \vdots & \ddots 
        \end{pmatrix}.
    \]
    This matrix is identical to the matrix $A$ used in \citet[Theorem~3.15]{Bubeck15}. Following the derivation in \citet[Theorem~3.15]{Bubeck15}, we know that the solution to our problem is 
    \[
        w^*_i = \left( \frac{\sqrt{ \beta/(N \alpha) } - 1}{ \sqrt{ \beta/(N \alpha) } + 1 }  \right)^i.
    \]
    With out loss generality, we assume that the initial point is $w^{(0)} = 0$. Let $K_t = \{ i \in \mathbb{N}_+ ~\mid~ w^{(t)}_i \neq 0 \}$. By the definition of $K_t$, we know that
    \begin{align}
        \| w^{(t)} - w^* \|^2 \geq \sum_{i= K_t+1}^{\infty} ( w^*_i )^2.  \label{eq:lowerBoundReminder}
    \end{align}
    By strong convexity, we further know that
    \[
        F(w^{(t)}) - F(w^*) \geq \frac{ N \alpha }{2} \| w^{(t)} - w^* \|^2 \geq \frac{ N \alpha }{2} \sum_{i= K_t+1}^{\infty} ( w^*_i )^2.
    \]
    Therefore we only need to have a lower bound for the term $\sum_{i= K_t+1}^{\infty} ( w^*_i )^2$ under the random participation black-box procedure.
    
    At $t=0$, $K_t = 0$ by the initialization. When $t=1$, $K_t=1$ only when the first client participates and $K_t = 0$ otherwise. Therefore $K_t = 1$ with probability $\frac{\tau}{N}$ and $0$ with probability $1 - \frac{\tau}{N}$. By the construction of the black-box procedure in \Cref{sec:lowerBound}, we can have the same conclusion for $t > 1$: at time $t$, $K_t = K_{t-1}+1$ with probability $\frac{\tau}{N}$ and stay unchanged with probability $1- \frac{\tau}{N}$. Therefore, $K_t$ follows the binomial distribution
    \[
        \Pr[ K_t = i ] = \binom{t}{i} \left( \frac{\tau}{N} \right)^i  \left( 1 - \frac{\tau}{N} \right)^{t-i} \qquad \forall i = 0, 1, \ldots, t.
    \]
    Now we are ready to bound \cref{eq:lowerBoundReminder}, let $\delta = \left( \frac{\sqrt{ \beta/(N \alpha) } - 1}{ \sqrt{ \beta/(N \alpha) } + 1 } \right)^2 $, then
    \begin{align*}
        \mE \left[ \sum_{i= K_t+1}^{\infty} ( w^*_i )^2 \right] &~=~ \mE \left[ \delta^{K_t + 1} \frac{1}{1 - \delta} \right]  \\
        &~=~ \frac{\delta}{1-\delta} \mE \left[ \delta^{K_t} \right]  \\
        &~=~ \frac{\delta}{1-\delta} \mE \left[ \exp ( \ln(\delta) K_t ) \right]  \\
        &~=~ \frac{\delta}{1-\delta} \left( \frac{\tau}{N} \delta + 1 -  \frac{\tau}{N} \right)^t \quad  \text{(By the moment generating function of $K_t$)} \\
        &~=~ \frac{\delta}{1-\delta} \left( 1 -  \frac{\tau}{N} \frac{4}{ \sqrt{ \beta/(N \alpha) } + 2 + \sqrt{ (N \alpha)/\beta }  } \right)^t \\
        &~\geq~ \frac{\delta}{1-\delta} \left( 1 -  \frac{\tau}{N} \frac{4}{ \min\{ \sqrt{ \beta/(N \alpha) } , 4\} } \right)^t \\
        &~=~ \frac{\delta}{1-\delta} \left( 1 -  \frac{\tau}{N} \min\left\{ \frac{4 \sqrt{N} }{ \sqrt{ \beta/\alpha }  }, 1 \right\} \right)^t.
    \end{align*}
    The above finished the proof for the case $\alpha < \beta/N$.
    
    When $\alpha \geq \beta /N$, we use similar construction. Let $m = \lfloor (1-\epsilon)(\beta)/\alpha \rfloor$ ($\epsilon$ is arbitrarily close to 0); $m$ is the largest integer such that $\beta > m \alpha$, we construct $m$ functions $\tilde{f}_i$'s instead of construct $N$ functions such that
    \[
        \tilde{f}_i(w) = \frac{\beta - m \alpha}{8} \left( w^T M^{(i)} w - \mathbf{1}_{i=1} \cdot 2 \langle e_1, w \rangle \right) + \frac{\alpha}{2} \| w \|^2 \qquad \forall i \in [m],
    \]
    where $M^{(i)}$'s are similar to previous construction except that we substitute $n$ with $m$ now. We partition the $n$ functions $f_i$'s into $m$ blocks $\Bscr_{i}, i=1,2,\ldots,m$, where
    \begin{align*}
        \Bscr_1 &~=~ \{1,2,\ldots, \lfloor n/m \rfloor \}, \\
        \Bscr_2 &~=~ \{ \lfloor n/m \rfloor+1,\lfloor n/m \rfloor+2,\ldots, 2\lfloor n/m \rfloor \}, \\
        & \dots \\
        \Bscr_m &~=~ \{ (m-1)\lfloor n/m \rfloor+1, (m-1)\lfloor n/m \rfloor+2,\ldots, n \}.
    \end{align*} 
    For any $i \in [N]$, we let
    \[
        f_i = \frac{1}{|\Bscr_j|} \tilde{f}_j  \quad \text{if $i \in \Bscr_j$}.
    \]
    It is also easy to verify that $f_i$'s are $\beta$-smooth and $\alpha$-strongly convex. Then we follow exactly the same argument of the case when $\alpha < \beta /N$. The only difference is that now $K_t$ has probability at most $\frac{ \tau \lceil N/m \rceil }{N}$ to increment by one in each iteration. Let let $\delta = \left( \frac{\sqrt{ \beta/(m \alpha) } - 1}{ \sqrt{ \beta/(m \alpha) } + 1 } \right)^2 $, the same argument gives us
    \begin{align*}
        \mE \left[ \sum_{i= K_t+1}^{\infty} ( w^*_i )^2 \right] &~\geq~ \frac{\delta}{1-\delta} \left( 1 - \min\left\{ \frac{\tau \lceil N/m \rceil }{N}  \frac{4 \sqrt{m} }{ \sqrt{ \beta/\alpha }  }, 1 \right\} \right)^t \\
        &~\geq~ \frac{\delta}{1-\delta} \left( 1 - \min\left\{   \frac{8 }{ \sqrt{m} \sqrt{ \beta/\alpha }  }, 1 \right\} \right)^t  \\
        &~\geq~ \frac{\delta}{1-\delta} \left( 1 - \min\left\{   \frac{8 \sqrt{2} }{   \beta/\alpha  }, 1 \right\} \right)^t  \quad \text{(By $ \frac{\beta}{2\alpha} \leq  m \leq \frac{\beta}{\alpha}$)}.
    \end{align*}
    The above finished the proof for the case when $\alpha \geq \beta /N$. Combine the results from the two cases, we finish the proof for \Cref{thm:lowerBound}.

\end{proof}

\end{document}